\documentclass[11pt]{article}

\usepackage[english]{babel}
\usepackage[a4paper,bindingoffset=0.2in,
            left=0.7in,right=0.7in,top=1in,bottom=1in,
            footskip=.25in]{geometry}

\usepackage{comment}
\usepackage{cite}
\usepackage{enumerate}
\usepackage{amsmath}
\usepackage{amssymb}
\usepackage{amsfonts} 
\usepackage{amsthm} 
\usepackage{bm}
\usepackage{pgf}
\usepackage{caption,subcaption}

\def\alphabet{abcdefghijklmnopqrstuvwxyzABCDEFGHIJKLMNOPQRST123456789}
\renewcommand{\vec}[1]{
	\IfSubStr{\alphabet}{#1}{
		\ensuremath{\mbf{\MakeLowercase{#1}}}
	}{
		\ensuremath{\bm{\MakeLowercase{#1}}}
	}
}

\def\R{\mathbb R}

\def\E{\mathbb E}

\def\N{\mathbb N}

\def\sign{\mathrm{sign}}

\renewcommand\epsilon{\varepsilon}

\newcommand{\vertiii}[1]{{\left\vert\kern-0.25ex\left\vert\kern-0.25ex\left\vert #1 
    \right\vert\kern-0.25ex\right\vert\kern-0.25ex\right\vert}}

\def\1{\vec{1}}

\def\vq{{\vec{q}}}

\def\vv{{\vec{v}}}
\def\vw{{\vec{w}}}
\def\vx{{\vec{x}}}
\def\vy{{\vec{y}}}
\def\vz{{\vec{z}}}

\def\mA{{\bm{A}}}

\def\mI{{\bm{I}}}

\def\mM{{\bm{M}}}

\def\mW{{\bm{W}}}
\def\mX{{\bm{X}}}
\def\mY{{\bm{Y}}}

\def\mPhi{{\bm{\Phi}}}

\def\mOmega{{\bm{\Omega}}}

\newcommand{\mbf}{\mathbf}

\newcommand{\mcl}{\mathcal}

\def\R{\mathbb{R}}

\def\E{\mathbb{E}}

\def\N{\mathbb{N}}

\def\d{\mathrm{d}}

\newcommand*{\norm}[1]{\left\|#1\right\|}

\newcommand*{\paran}[1]{\left(#1\right)}
\newcommand*{\card}[1]{\left|#1\right|}

\newcommand*{\prob}[1]{\mathbb{P}}

\def\defeq{:=}

\def\d{\mathrm{d}} 

\def\defeq{:=}

\def\d{\mathrm{d}}

\def\trainS{\mathcal{S}} 
\newcommand{\hypS}{\mathcal{H}} 

\newcommand{\hypspacelong}{\mcl{H}^L}

\def\datadist{\mathcal D}

\def\dimx{N} 
\def\dimy{n} 
\def\dimS{m} 
\def\dimL{L} 

\def\Bin{B_{\text{in}}}
\def\Bout{B_{\text{out} }}

\def\Lemp{\hat{\mathcal{L}}} 
\def\Ltrue{{\mathcal{L}}} 

\def\GE{\mathrm{GE}} 
\newcommand{\radem}{\mathcal{R}} 
\def\cover{\mathcal{N}}

\usepackage{setspace}        

\usepackage[printonlyused]{acronym}
\acrodef{ISTA}{Iterative soft thresholding algorithm}
\acrodef{LISTA}{Learned Iterative soft thresholding algorithm}

\usepackage{dcolumn,amsthm}

\renewenvironment{proof}[1][Proof]{\noindent\textit{#1. } }{\hfill$\square$}

\newtheoremstyle{theorem}{6pt}{6pt}{\rm}{}{\sffamily}{ }{ }{}
\theoremstyle{theorem}
\newtheorem{theorem}{\sc Theorem}[section]

\newtheoremstyle{lemma}{6pt}{6pt}{\rm}{}{\sffamily}{ }{ }{}
\theoremstyle{lemma}
\newtheorem{lemma}{\sc Lemma}[section]

\newtheoremstyle{example}{6pt}{6pt}{\rm}{}{\sffamily}{ }{ }{}
\theoremstyle{example}

\newtheoremstyle{corollary}{6pt}{6pt}{\rm}{}{\sffamily}{ }{ }{}
\theoremstyle{corollary}
\newtheorem{corollary}{\sc Corollary}[section]

\newtheoremstyle{definition}{6pt}{6pt}{\rm}{}{\sffamily}{ }{ }{}
\theoremstyle{definition}

\newtheoremstyle{remark}{6pt}{6pt}{\rm}{}{\sffamily}{ }{ }{}
\theoremstyle{remark}
\newtheorem{remark}{\sc Remark}[section]

\newtheoremstyle{approximation}{6pt}{6pt}{\rm}{}{\sffamily}{ }{ }{}
\theoremstyle{approximation}

\newtheoremstyle{scheme}{6pt}{6pt}{\rm}{}{\sffamily}{ }{ }{}
\theoremstyle{scheme}

\usepackage{authblk}

\title{Generalization Error Bounds for Iterative Recovery Algorithms Unfolded as Neural Networks}

\author[1]{Ekkehard Schnoor \thanks{schnoor@mathc.rwth-aachen.de}}
\author[2]{Arash Behboodi \thanks{arash.behboodi@ti.rwth-aachen.de}}
\author[1]{Holger Rauhut \thanks{rauhut@mathc.rwth-aachen.de}}
\affil[1]{Chair for Mathematics of Information Processing, RWTH Aachen University}
\affil[2]{Institute for Theoretical Information Technology, RWTH Aachen University}

\begin{document}

\maketitle

\begin{abstract}
{Motivated by the learned iterative soft thresholding algorithm (LISTA), we introduce a general class of neural networks suitable for sparse reconstruction from few linear measurements. 
By allowing a wide range of degrees of weight-sharing between the layers, we enable a unified analysis for very different neural network types, ranging from recurrent ones to networks more similar to standard feedforward neural networks. Based on training samples, via empirical risk minimization we aim at learning the optimal network parameters and thereby the optimal network that reconstructs signals from their low-dimensional linear measurements. We derive generalization bounds by analyzing the Rademacher complexity of hypothesis classes consisting of such deep networks, that also take into account the thresholding parameters. We obtain estimates of the sample complexity that essentially depend only linearly on the 
number of parameters and on the depth.
We apply our main result to obtain specific generalization bounds for several practical examples, including different algorithms for (implicit) dictionary learning, and convolutional neural networks.}
\end{abstract}

\section{Introduction}
Deep neural networks with their strength in learning useful representations and features from data have shown remarkable performance in many tasks in computer vision and natural language processing among others. These multi-layer architectures can be seen as stack of features organized in a layer-wise way.  
Iterative optimization algorithms such as the iterative soft thresholding algorithm (ISTA) \cite{daubechies2004iterative} can be unfolded as layers of a neural network with skip connections and recurrent structure. 
Based on this observation we can optimize the parameters of the iterative algorithms as a multi-layer network to achieve better convergence and performance. For example, the learned
iterative soft thresholding algorithm (ISTA) method \cite{gregor2010learning} optimizes an unfolded iterative soft thresholding algorithm (ISTA) by replacing all linear operations with learnable linear layers and shows promising performance. In such a way, the reconstruction can adapt
potentially much better to the class of (training) signals specific to the application at hand than a generic optimization approach.   


In this paper, we introduce and analyze a broad class
of neural networks, which we call deep iterative recovery networks, that can be used for representation learning and sparse recovery tasks from the viewpoint
of statistical learning theory and provide bounds
for the generalization error for corresponding
hypothesis classes of such networks.


As a practical application, one may think of reconstructing images from measurements taken by a medical imaging device. Instead of only trying to reconstruct the image, we would like to implicitly learn also a meaningful representation system which is adapted to the image class of interest, and leads to good generalization (e.g., when taking measurements of new patients). This  approach is in general useful for solving inverse problems in a data-driven way \cite{arridge2019solving, gottschling2020troublesome}.
As such it is closely related to dictionary learning methods
\cite{gribonval2015sample,vainsencher2011sample} such as K-SVD \cite{Aharon-2006,sc14}, which however do not take into account the reconstruction method but only aim at learning a sparse representation system for the type of data at hand.


While so far generalization of neural networks has been studied mostly in the context of classification using feed-forward neural networks, see e.g.\ \cite{arora2018stronger,bartlett_spectrally-normalized_2017,jiang2019fantastic,neyshabur2017pac,neyshabur2017exploring}, our case studied falls into the class of regression problems which has received far less attention so far from the perspective of generalization. 
Our main contribution is a unified 
analysis of the generalization error for a large class of neural networks, including both feedforward and recurrent neural networks. Note that generalization bounds for
recurrent networks are mostly neglected in the literature, and while they
are difficult to train in general \cite{pascanu2013difficulty}, such problems do not occur in our setup. 

Our analysis considers various degrees of weight-sharing between layers. In this way we obtain a 
unified treatment both for networks with a moderate number and with a large number of parameters.
Our generalization bounds 
are shown by estimating the Rademacher complexity of hypothesis classes consisting of such deep networks via a generalization \cite{maurer_vector-contraction_2016} of Talagrand's contraction principle \cite{ledoux_probability_2011}
and Dudley's integral, and in particular via the covering numbers involved. For the latter, we derive bounds on
the Lipschitz constants of the networks in terms of the
network parameters.
We believe that the techniques presented are of independent interest far beyond the particular problem studied here, e.g., for a theoretical investigation of related iterative schemes, general regression problems using neural networks, and in particular autoencoders and recurrent neural networks.

The paper is structured as follows.
In Section~\ref{sec:DSTmodel} we introduce our deep iterative recovery networks, 
and formulate the corresponding machine learning problem. 
Our main result that bounds the generalization error for the setup introduced before
is stated in Section~\ref{sec:main_result}, where we also give various examples for interesting special cases with
the resulting generalization bound. These examples
include weight-sharing, convolutional networks and
special cases with certain fixed parameters.
In Section~\ref{sec:generalization} we describe our approach
to solve problem at hand by bounding the Rademacher complexity using Dudley's integral. 
Section~\ref{sec:proof} provides the detailed proof of our main result.
Finally, in Section~\ref{sec:numerical} we present the results of our numerical experiments and compare it with our theoretical findings in the section before.

\section{Unfolded Iterative Algorithms}
\label{sec:DSTmodel}

\subsection{Motivation: Learning a dictionary implicitly by training a decoder}
\label{subsection:motivation_example}

As a motivating example, let us first recall our previous work
\cite{behboodi2020generalization}, which we are going to vastly generalize in this paper.
There, we consider a class of signals $\vx \in \R^\dimx$ which are sparsely representable with respect to a dictionary $\mPhi_0 \in \R^{\dimx \times \dimx}$. In other words, for each $\vx$ there is an (approximately) sparse vector $\vz \in \R^\dimx$ such that $\vx = \mPhi_0 \vz$. The dictionary $\mPhi_0$ is assumed to be unknown. We are given a linear observation $\vy = \mA \vx \in \R^\dimy$ where $\mA \in \R^{n \times N}$ is a known measurement matrix and the task is to reconstruct $\vx$ from $\vy$. 
In order to exploit sparsity in this compressive sensing problem, 
we would like to learn a dictionary $\mPhi$ (ideally $\mPhi_0$) 
suitable for decoding purpose based on a training sequence\footnote{Note that the notation $(\vx_i, \vy_i)$ is often used e.g.\ for classification problems with input features $\vx_i$ and output labels $\vy_i$. For simplicity, we stick to this notation, even though the other order, i.e. $(\vy_i, \vx_i)$, would be arguably more suitable, since our decoder takes as the input the measurements $\vy_i$ and outputs the reconstruction $\vx_i$. 
}
$\trainS \defeq \paran{(\vx_i,\vy_i)}_{i = 1, \dots, m}$. For the purpose of providing generalization bounds we will assume later on that the samples $(\vx_i,\vy_i)$ are drawn
i.i.d.\ from a distribution $\datadist$, which
is unknown a priori and, for simplicity, such that 
$\vy_i = \mA \vx_i$ for each draw.

A common approach for the reconstruction (with known $\mPhi$) is $\ell_1$-minimization, i.e., given $\vy$ one computes
a coefficient vector $\vz^\sharp$ as the minimizer of
\begin{equation}\label{l1-opt}
\min_{\vz \in \R^\dimx} \frac{1}{2} \|\mA \mPhi \vz - \vy\|_2^2 + \lambda \|\vz\|_1
\end{equation}
and then outputs $\vx^\sharp = \mPhi \vz^\sharp$ as reconstruction. Here, $\lambda > 0$ is a suitable regularization parameter.
The iterative soft thresholding algorithm (ISTA) \cite{daubechies2004iterative} is an algorithm that solves the above optimization problem (under mild conditions).
We formulate $L$ iterations ISTA directly as a neural network with $L$ layers as follows. 
The first layer is defined by $f_1(\vy) \defeq S_{\tau\lambda}(\tau (\mA\mPhi)^\top\vy)$, where $S_\lambda$ (applied entry-wise) 
is the shrinkage operator defined as
\begin{align}\label{eq:soft_thresholding_operator}
    S_\lambda: \R  \to \R, \qquad  x & \mapsto     
    \begin{cases}
      0 & \mathrm{if} \; |x| \; \leq \lambda, \\
      x - \lambda \sign(x)       & \mathrm{if} \; |x| \; > \lambda,
    \end{cases}
\end{align}
which can also be expressed in closed form as $S_\lambda(x) = \sign(x) \cdot \max(0,|x| - \lambda)$. For $l>1$, the layer is given by
\begin{align}
f_l(\vz) 
& \defeq S_{\tau\lambda}
    \left[\vz + \tau (\mA\mPhi)^\top(\vy-(\mA\mPhi)\vz)\right] \label{def:ISTA}\\
& = S_{\tau\lambda}\left[ \left(\mI - \tau \mPhi^\top \mA^\top \mA \mPhi \right )\vz          + \tau (\mA \mPhi )^\top \vy \right] \notag,
\end{align}
which can be interpreted as a layer of a neural network with weight matrix 
$\mI - \tau \mPhi^\top \mA^\top \mA \mPhi$, bias $\tau (\mA \mPhi )^\top \vy$ and activation function $S_{\tau\lambda}$. Recalling that the optimization problem \eqref{l1-opt} gives a coefficient vector we still need to apply the dictionary $\mPhi$ in order to obtain a reconstruction so that the decoder, i.e, the final neural network, parameterized by $\mPhi \in O(N)$, takes the form 
\begin{equation}
f_\mPhi^L(\vy) = 
\mPhi \circ 
f_L\circ f_{L-1}\dots \circ f_1(\vy).
\label{eq:f_Phi^l}
\end{equation}
Note that for $l>1$, all layers $f_l$ coincide as functions on $\R^\dimx$, leading to weight sharing across all layers
and to the fact that this is actually a special case of a recurrent network.

In order to adapt this network, i.e., the parameter 
$\mPhi$, to training data, one introduces the hypothesis
class $\mathcal{H} = \{ f_\mPhi^L : \mPhi \in O(N)\}$ (up to an additional transform after the last layer like the function $\sigma$ below in \eqref{eq:function_sigma})
and optimizes an empirical loss function with respect to $\mPhi$. In our previous work \cite{behboodi2020generalization}, we have derived bounds
on the generalization error in terms of the number of training samples, the dimension $\dimx$, 
the number $\dimy$ and the number of layers by bounding
the Rademacher complexity of $\mathcal{H}$.
For all details and proofs we again refer to \cite{behboodi2020generalization}.

\subsection{A general setup}
\label{subsec:general_setup}

We will now introduce a considerably more general setting, that goes far beyond the particular example above, but still contains it as a special case. 
We abandon the assumption of necessary weight-sharing between all layers. More precisely, the weight sharing can happen in any possible order, i.e., between any arbitrary number of layers, appearing at any position in the neural network - in particular, weight sharing is not only possible among subsequent layers. (No weight sharing is also included.)
Furthermore, we allow various additional (trainable) parameters, and include additional $1$-Lipschitz operations after each soft thresholding step, such as pooling operations. 

Formally, for $L$ being the number of layers in the decoder, we introduce $J \leq L+1$ bounded parameter sets
%
%
\[
\mcl{W}^{(1)} \subset \R^{k_1}
, \dots ,
\mcl{W}^{(J)} \subset \R^{k_J},
\qquad 
k_1, \dots, k_J \in \N,
\]
where $\R^{k_j}$ is equipped with a norm $\| \, . \, \|^{(j)}$. For each layer $l = 1,\hdots,L+1$ (including a final transform after the last layer), we introduce Lipschitz continuous mappings $B_l$ (often linear) that provide the parameterization
of a matrix $B_l(\vw^{(j)}) \in \R^{n \times n_{l-1}}$ using a parameter 
$\vw^{(j)} \in \mcl{W}^{(j)}$, where $j = j(l)$ corresponds to the parameter set associated to the $l$-layer:
\begin{equation}
B_l : \mcl{W}^{(j(l))} \to \R^{n \times n_{l-1}},
\qquad 
\vw^{(j)} \mapsto B_l(\vw^{(j)}).
\label{eq:B_l_Lipschitz}
\end{equation}
Note that if $J = 1$, then all layers share the same weights; if $J=L+1$, there is no weight sharing and all layers, and the final transform after the last layer, 
have different underlying parameters.
If $l$ is either clear from the context, or not relevant, we may omit it in $j(l)$ and simply write $j$.
If $j(l) = j(l')$ for any two different layers $l \neq l'$, the two layers share the same weights. Note that even in this situation still it may be that $B_l \neq B_{l'}$, since already the 
involved dimensions $n_{l-1}$ and $n_{l'-1}$ may be different 
- this means that even if layers share the same \emph{underlying parameters}, the \emph{parameterizations} in the sense of the mappings $B_l$ and $B_{l'}$ may still be different. (We typically denote the index refering to the parameter set as an \emph{upper} index, and the index referring to the layer number as a 
\emph{lower} index.) Let us also remark that $\mcl{W}^{(i)} = \mcl{W}^{(k)}$ is possible even when $i \neq k$.

To make the Lipschitz assumption precise, we require that for each $l \in [L + 1]$, 
there exists a constant $D_l > 0$, such that
\begin{equation}
\| B_l(\vw_1) -  B_l(\vw_2) \|_{2 \to 2} 
\leq  
D_l \| \vw_1 -  \vw_2 \|^{(j(l))}
\qquad \forall \, \vw_1, \vw_2 \in \mcl{W}^{(j(l))}.
\label{eq:definition_D_l}
\end{equation}

In order to introduce the network architecture, let $\mI_{k}$ denote the $k \times k$ identity matrix, for some $k \in \N$, and $S_{\lambda}$ the soft thresholding operator \eqref{eq:soft_thresholding_operator} acting componentwise.
Further, we will use a  $1$-Lipschitz operation $P_l : \R^{n_{l-1}} \to \R^{n_l}$ such as pooling, which satisfies
\[
\| P_l(\vz) \|_2 \leq \| \vz \|_2 \qquad \forall \, \vz \in  \R^{n_{l-1}}.
\]
Then $P_l \circ S_{\tau_l \lambda_l}$ is also $1$-Lipschitz and norm contractive.
(In many scenarios with ${n_{l-1}} = {n_l}$, $P_l$ will simply be the identity; see also Remark~\ref{rem:comment_parameter_dep_sharing} and the examples
in Section~\ref{sec:examples}.)

For $l=1,\hdots,L$, and dimension (width) parameters
$n_0, \hdots, n_L$, we then define the layer $f_l: \R^{n_{l-1}} \times \R^n \to \R^{n_l}$
as
\[
f_l \left( \vz, \vy \right)
=
P_l S_{\tau_l \lambda_l}
\left[ \left(\mI_{n_{l-1}} - 
    \tau_l B_l(\vw^{(j(l))})^\top  B_l(\vw^{(j(l))}) 
    \right )\vz 
     + \tau_l  B_l(\vw^{(j(l))})^\top \vy \right],
\]
with parameter vector $\vw^{(j(l))} \in \mcl{W}^{(j(l))}$, stepsize $\tau_l$,
threshold $\lambda_l$. The input vector $\vy \in \R^n$ may be 
$\vy= \mA \vx \in \R^{n}$ for some (a priori unknown) vector $\vx \in \R^N$ in a compressive sensing
scenario, but our setup allows more general regression tasks.
The vector $\vz$ will be initialized as $\mathbf{0}$ for the input of the first layer; afterwards it will be the output of the previous layer (see below). 
The stepsize $\tau_l > 0$ and the threshold $\lambda_l > 0$ in the soft thresholding activation function can be either trainable parameters, or fixed all the time. In the simplest case $\tau_l = \tau, \lambda_l = \lambda > 0$ are fixed and the same in each layers.  

After the final layer $f_L$, we apply another
linear transform $B_{L+1}(\vw^{j(L+1)})$ followed by some function $\sigma : \R^{n_{L+1}} \to \R^{n_{L+1}}$,
\[
g_{L+1} : \R^{n_L} \to \R^{n_{L+1}}, \quad g_{L+1} 
= \sigma \circ B_{L+1}(\vw^{j(L+1)}).
\]
For reconstruction tasks, the function $g_{L+1}$ projects the sparse representation onto the ambient space and controls the output norm.
In fact, the function $\sigma$ is assumed to be norm-contractive and norm-clipping as well as be $1$-Lipschitz, i.e.,
\begin{equation}
\label{ABHRES:sigma:bound}
\| \sigma( \vx) \|_2 \leq \min\{\|  \vx \|_2, B_{\operatorname{out}}\}
\qquad \mathrm{and} \qquad 
\| \sigma( \vx_1) - \sigma( \vx_2)\|_2 
\leq 
\|  \vx_1 -  \vx_2\|_2
\end{equation}
for any $ \vx$ and  $ \vx_1,  \vec{x}_2 \in \R^{n_L}$
and some fixed constant $B_{\operatorname{out}} > 0$.
The technical reasons behind introducing $\sigma$ will become apparent later in the proofs in Section~\ref{sec:proof}.
A typical choice for $\sigma$ satisfying all requirements is
\begin{align}
        \sigma: \R^{n_{L+1}} \to \R^{n_{L+1}}, \qquad  \vx & \mapsto     
    \begin{cases}
      \vx & \mathrm{if} \;  \| \vx\|_2 \leq B_{\operatorname{out}}, \\
      B_{\operatorname{out}} \frac{ \vx}{\| \vx\|_2} & 
            \mathrm{if} \;  \| \vx\|_2 > B_{\operatorname{out}},
    \end{cases}
    \label{eq:function_sigma}
\end{align}
The role of $\sigma$ is to push the output of the network inside the $\ell_2$-ball of radius $B_{\operatorname{out}}$, which in many applications is approximately known. The prior knowledge about the range of outputs can improve the reconstruction performance and generalization \cite{Wu2020Sparse}. {The constant $B_{\operatorname{out}}$ may} be simply chosen to be equal to $B_{\operatorname{in}}$, for instance.


Note that for the first layer's input we have $n_0 = n$, i.e., the number of measurements. A typical choice
for the last layer's dimension is $n_{L+1} = N$, which corresponds to the setting of reconstruction problems; but note that our framework allows to consider different situations.
Let us introduce the compact notation
\begin{equation*}
\mcl{W} \defeq
\mcl{W}^{(1)}  \times \dots \times \mcl{W}^{(J)}
\subset \R^{k_1} \times \dots \times \R^{k_J} =: \mcl{X}
\end{equation*}
for the set of $K$-dimensional weights 
$\mW  = (\vw^{(1)}, \dots \vw^{(J)}) \in \mcl{W}$,
where $K$ is the sum of the individual dimensions $k_j = \dim \mcl{W}^{(J)}$, i.e.,
\begin{equation}
    K \defeq k_1 + \dots + k_J.
    \label{eq:def_K}
\end{equation}
In order to allow for learnable stepsizes and thresholds
we introduce the set $\mathcal{T} \subset \R_{>0}^L$
of stepsize vectors $\bm{\tau}  = (\tau_1, \dots, \tau_L)$ and the set $\Lambda \subset \R_{>0}^L$
of thresholding vectors $\bm{\lambda}  = (\lambda_1, \dots, \lambda_L)$. 
Then we define $f_{\mW, \bm{\tau}, \bm{\lambda}}^L$ to be the concatenation of all layers $f_l$, 
\begin{equation*}
        f_{\mW, \bm{\tau}, \bm{\lambda}}^L (\vy) 
\defeq  f_L ( \dots f_2 ( f_1 (\mathbf{0}, \vy ), \vy)\dots),
\label{eq:def_f^L}
\end{equation*}

and the neural network -- also called decoder -- is obtained after an application of $g_{L+1}$,


\begin{equation}
h(\vy) = h_{\mW, \bm{\tau}, \bm{\lambda}}^L \defeq 
        g_{L+1} \circ f_{\mW, \bm{\tau}, \bm{\lambda}}^L (\vy) 
=       g_{L+1} (f_L ( \dots f_2 ( f_1 (\mathbf{0}, \vy ), \vy)\dots)).
\label{eq:def_hypothesis}
\end{equation}

The fact that the input $\vy$ is entered directly into each of the layers in addition to the input from the previous layers, may be interpreted as the network having so-called skip connections.

For the investigations in the following sections it will be convenient to view the parameter sets 
as subsets of normed spaces. The set $\mcl{W}$ is contained in the $K$-dimensional 
product space $\mcl{X} = \R^{k_1} \times \cdots \times \R^{k_J}$, which we equip with the norm
\begin{equation}
\|\mW\|_\mcl{X} \defeq \max_{j = 1, \dots, J} \| \vw^{(j)}\|^{(j)}
\qquad \text{for} \qquad 
\mW = \left( \vw^{(1)}, \dots, \vw^{(J)} \right) \in \mcl{X},
\label{eq:max_norm_X}
\end{equation}
where we recall that $\| \cdot \|^{(j)}$ is the norm on $\R^{k_j}$ 
used in \eqref{eq:definition_D_l}.
Denoting $B_{\|\cdot\|_\infty}^L = \{\bm{\tau} \in \R^L : \|\bm{\tau}\|_\infty \leq 1\}$ the unit $\ell_\infty$-ball, we assume that the set $\mathcal{T}$ of stepsizes and the $\Lambda$ of thresholds are contained in
shifted $\ell_\infty$-balls of radii $r_1$ and $r_2$, i.e.,
\begin{equation}
\label{thresholds-inclusion}
\mathcal{T} \subset \bm{\tau_0} + r_1B_{\|\cdot\|_\infty}^L \qquad
\Lambda \subset \bm{\lambda}_0 + r_2B_{\|\cdot\|_\infty}^L.
\end{equation}
Setting $r_1 = r_2 = 0$ corresponds to the case of fixed stepsizes and thresholds while choosing $r_1, r_2 > 0$ corresponds to learned stepsizes and thresholds. The above conditions require that
$\tau_j \in [\tau_{0,j} - r_1,\tau_{0,j} + r_1]$ for all $\bm{\tau} \in \mathcal{T}$ and
$\lambda_j \in [\lambda_{0,j} - r_1,\lambda_{0,j} + r_2]$ for all $\bm{\lambda} \in \Lambda$.
Recalling that $\mcl{W}$ is assumed to be bounded,
we can introduce the parameters
\begin{equation}
B_\infty 
\defeq 
\mathop{\sup}_{\substack{\mW \in \mcl{W} \\ l\in [L+1]}}
\left \| B_l(\vw^{(j(l))}) \right\|_{2\to 2}, \quad
W_\infty \defeq \mathop{\sup}_{\substack{\mW \in \mcl{W}}} \|\mW\|_{\mcl{X}},
\quad
\tau_\infty 
\defeq
\sup_{\bm{\tau} \in \mcl{T}} \|\bm{\tau}\|_\infty,
\quad
\lambda_\infty 
\defeq
\sup_{\bm{\lambda} \in \Lambda} \|\bm{\lambda}\|_\infty.
\label{eq:standing_assumptions}
\end{equation}
Note that if the mappings $B_l$ are linear then
\begin{align}
    B_\infty \leq W_\infty \max_{l\in [L + 1]} D_l \label{eq:BWD-bound}
\end{align}
Moreover, $\tau_\infty \leq \|\bm{\tau}_0\|_\infty + r_1$ and
$\lambda_\infty \leq \|\bm{\lambda}_0\|_\infty + r_2$.

\begin{remark}\label{rem:comment_parameter_dep_sharing}
Let us add a some comments to motivate the setup above.
Allowing weight-sharing between layers can be easily motivated, for instance, through the dictionary learning problem discussed as our introductory example in Section \ref{subsection:motivation_example}.
However, weight-sharing between the stepsizes and thresholds seems less realistic, which is why we consider different
$\tau_l$ and $\lambda_l$ for each layer, $l \in [L]$. It is possible to generalize this even further by 
training these parameters entrywise. However, to prevent the presentation from becoming even more technical, we focus on the problem at hand.
\end{remark}

\subsection{Hypothesis class and loss function}

Using the concepts and notation introduced above, given a number $L$ of layers, we define our hypothesis space as the parameterized set of all $h$ (see \eqref{eq:def_hypothesis} for the definition of $h$), i.e.
\begin{equation}
\hypspacelong
:= 
\{ 
h : \R^n \to \R^{n_{L+1}} \;|\; 
h = h_{\mW, \bm{\tau}, \bm{\lambda}}^L, \,
\mW \in \mcl{W}, \, \bm{\tau} \in \mcl{T}, \, \bm{\lambda} \in \Lambda
\},
\label{eq:hypothesis_space_general}
\end{equation}
Based on the training sequence $\trainS = ((\vx_i,\vy_i))_{i=1,\hdots,\dimS}$ and given the hypothesis $\hypspacelong$,
a learning algorithm yields a function $h_\trainS\in\hypspacelong$ that aims at
inferring $\vx$ from $\vy$, which in the compressive
sensing scenario amounts to
reconstructing $\vx$ 
from the measurements $\vy = \mA \vx$. 
Given a loss function $\ell: \hypspacelong \times \R^\dimx \times \R^\dimy  \to \R$, the
empirical risk is the reconstruction error on the training sequence,
i.e., the difference between $\vx_i$ and $\hat{\vx}_i =  h_\trainS(\vy_i)$, that is
\[
\Lemp(h) = \frac{1}{\dimS}\sum_{j=1}^\dimS \ell(h,\vx_j,\vy_j).
\]
Different choices for the loss function $\ell$ are possible. Our focus will be on the $\ell_2$-loss 
\begin{align}
\ell(h,\vx,\vy) = \| h(\vy)-\vx\|_2,
\label{eq:loss_function}
\end{align}
which simply measures the reconstruction error with respect to the $\ell_2$ norm. We note that we may also work with the commonly used squared $\ell_2$-norm instead of the unsquared $\ell_2$-norm by adjusting a few details in constants. Using the unsquared $\ell_2$-norm is slightly more convenient for us. 

The true loss, i.e., the risk of a hypothesis $h$ is accordingly defined as 
\[
\Ltrue(h) \defeq\E_{\vx,\vy\sim\datadist}\paran{\ell(h,\vx,\vy)}.
\]
The generalization error is defined as the difference between the empirical loss 
and the true loss of the hypothesis $h_{\trainS}$ learned based on the training sequence $\trainS$,
\[
\GE(h_\trainS) \defeq \card{\Lemp(h_\trainS) - \Ltrue(h_\trainS)}.
\]
(Note that some references denote the true loss $\Ltrue(h_\trainS)$ as the generalization error. However, the above definition is more convenient for our purposes.)
We use a Rademacher complexity analysis to bound the generalization error in the next section.

\subsection{Related Work}

The idea of interpreting gradient steps of iterative algorithms such as ISTA \cite{daubechies2004iterative} for sparse recovery as layers of neural networks has appeared in \cite{gregor2010learning} and has then become an active research topic, e.g., \cite{Wu2020Sparse,liu2018alista, chen2018theoretical, kamilov2016learning, mousavi2015deep, xin2016maximal}.
The present paper is another contribution in this line of work and can be seen as a direct follow-up to our previous work
\cite{behboodi2020generalization}. Both are characterized by studying LISTA-inspired networks from a generalization perspective, which has been neglected in the literature before.
Our previous work \cite{behboodi2020generalization} focusses on a comparably simple problem of learning a dictionary suitable for reconstruction and may serve as an introduction to the topic, containing many related references and also a short introduction to generalization of neural networks for classification problems. 
Instead, this paper studies a much more general framework aiming to capture many other models of practical interest. It contains the scenario studied in \cite{behboodi2020generalization} as a special case, but also other models studied before, such as a class of LISTA models that use convolutional dictionaries \cite{sreter_learned_2018}; see also Section~\ref{sec:examples}.

To our best knowledge, it provides the first generalization error bounds for all of them, apart from our own previous work \cite{behboodi2020generalization}. Thus, it will serve as a reference and baseline for comparison with future works.
Even though the basic proof methods are very similar to the ones used in \cite{behboodi2020generalization}, the derivations become clearly more involved by taking additional training parameters into accout, as well as the numerical experiments. Instead of novel algorithmic aspects, our contribution is to conduct a generalization analysis for a large class of recovery algorithms, which to the best of our knowledge has not been addressed in the literature before in this particular setting. Furthermore, our setup proposed here also includes general regression tasks apart from reconstruction. In this way, we connect this line of research with recent developments \cite{golowich_size-independent_2018, bartlett_spectrally-normalized_2017} in the study of generalization of deep neural networks.
Particularly, we use a similar framework to \cite{bartlett_spectrally-normalized_2017} by bounding the Rademacher complexity using Dudley's integral. However, the approach of \cite{bartlett_spectrally-normalized_2017} applies only to the use of neural networks for classification problems. The extension to our problem, which is a regression problem with vector-valued functions, involves additional technicalities requiring the generalized contraction principle for hypothesis classes of vector-valued functions from \cite{maurer_vector-contraction_2016}. Besides, we show linear dependence of the number of training samples with the dimension (number of free parameters), using techniques that are different from the ones in \cite{golowich_size-independent_2018}.
It is not straightforward to extend the result of \cite{golowich_size-independent_2018} to our case because we allow weight sharing between different layers of the thresholding networks.

The unfolded networks we consider here fall into the larger class of proximal neural networks studied in \cite{hahene19, hertrich2021convolutional, hasannasab2021correction}. 
Many other related works are in the context of dictionary learning or sparse coding: The central problem of sparse coding is to learn weight matrices for an unfolded version of ISTA. Different works focus on different parametrization of the network for faster convergence and better reconstructions. Learning the dictionary can also be implicit in these works.  
Some of the examples of these algorithms are recently suggested Ada-LISTA \cite{aberdam2020ada}, convolutional sparse coding \cite{sreter2018learned} learning efficient sparse and low-rank models \cite{sprechmann2015learning}. 
Like many other related papers, such as ISTA-Net \cite{zhang2018ista}, 
these methods are  mainly motivated by applications like inpainting \cite{aberdam2020ada}. Sample complexity of dictionary learning has been studied before in the literature \cite{grsc10,vainsencher2011sample,sc14,gribonval2015sample,pmlr-v80-georgogiannis18a}. The authors in \cite{vainsencher2011sample} also use a Rademacher complexity analysis for dictionary learning, but they aim at sparse representation of signals rather than reconstruction from compressed measurements and moreover, they do not use neural network structures.
Fundamental limits of dictionary learning from an information-theoretic perspective has been studied in \cite{6952232, 7378975}. Unique about our perspective and different to the cited papers is our approach for determining the sample complexity based on learning a dictionary (or generally, other parameters to enable good reconstruction) implicitly by training a neural network.

In case of weight sharing between all layers, the networks is a recurrent neural network. The authors of  \cite{dasgupta_sample_1996} derive VC-dimension estimates of recurrent networks for recurrent perceptrons with binary outputs. The VC-dimension of recurrent neural networks for different classes of activation functions has been studied in \cite{koiran_vapnik-chervonenkis_1998}. However, their results do not apply to our setup, since they focus on one-dimensional inputs and outputs, 
i.e. corresponding to just a single measurement in our compressive sensing scenario. 
Furthermore, VC dimension bounds are mainly suited for classification tasks, making an application to (and comparison with) our vector-valued regression problem difficult.

\subsection{Notation}

Before we continue with the main part of the paper, let us fix some notation. 
Besides the very standard notation for numbers etc., $\R_{>0}$ denotes the positive real numbers.
Vectors $\vv \in \R^n$ and matrices $\mA \in \R^{\dimS \times \dimx}$ are denoted with bold letters, unlike scalars $\lambda \in \R$.  We will denote the spectral norm by 
$\| \mA \|_{2 \to 2}$ and the Frobenius norm by $\| \mA \|_F$. 
$\mI_k$ denotes the $k \times k$ identity matrix; if the size is clear from the context, we simply write $\mI$.
The $N \times m$ matrix $\mX$ contains the data points, $\vx_1, \dots, \vx_\dimS \in \R^\dimx$, as its columns. As a short notation for indices we use 
$[m] := \{1, \dots, m\}$, e.g. $(\vx_i)_{i \in [m]} = (\vx_1, \dots, \vx_m)$.
Analogously $\mY \in \R^{\dimy \times \dimS}$ denotes the matrix collecting the measurements $\vy_1, \dots, \vy_\dimS\in \R^\dimy$ in its columns. To make the notation more compact, with a slight abuse of notation, for $f_\mPhi^L \in \mathcal{H}^L$, we denote by $f_\mPhi^L (\mY)$ the matrix whose $i$-th column is $f_\mPhi^L (\vy_i)$. 
Similarly, vertical stacking up of (column) vectors $\vx_1, \dots, \vx_m \in \R^N$ to form a
$N \times m$-matrix ist denoted by $(\vx_1 | \dots | \vx_m)$.
The unit ball of an $n$-dimensional normed space $(\mcl{V}, \|\, . \,\|_{\mcl{V}})$ is denoted by 
$B_{\|\cdot\|_{\mcl{V}}}^n := \{\vx \in \mcl{V}: \|\vx\| \leq 1\}$, or simply $B_{\mcl{V}}^n$. The covering number $\cover{ \left(\mcl{M}, d, \varepsilon \right)}$ of a metric space $(\mcl{M}, d)$ at level $\varepsilon$ equals the smallest number of balls of radius $\varepsilon$ with center contained in $\mcl{M}$ with respect to the metric $d$ required to cover $\mcl{M}$. When we consider subsets of normed spaces where the metric is induced by the norm, we write $\cover{ \left(\mcl{M}, \| \cdot \|, \varepsilon \right)}$. Furthermore, we have already introduced the hypothesis spaces $\hypspacelong$ in \eqref{eq:hypothesis_space_general}.
Instead, we write $\mathcal{H}$ if we refer to a general hypothesis space, e.g., when quoting general results from the machine learning literature.

\section{Main Result}
\label{sec:main_result}

In this section, we begin by stating our main result, providing a generalization bound for the problem described in the previous section. Furthermore, we  illustrate our result with several examples of practical interest, providing specific generalization error bounds for all these cases. 

\subsection{Bounding the Generalization Error}

Our main result uses the setup and the notation introduced in Section~\ref{subsec:general_setup}. 
In order to state it, we additionally introduce the following quantities,
where we recall that $B_\infty$, $W_\infty$, $\tau_\infty$ and $\lambda_\infty$ are defined in \eqref{eq:standing_assumptions}, the dimension $K$ of the parameter set of weights in \eqref{eq:def_K}, 
and $D_\infty$ in \eqref{eq:definition_D_l}. 
We set 
\begin{align}
    \alpha       
    & =  \sup_{l \in [L]} \; \sup_{\vw^{(j(l))} \in \mcl{W}^{(j(l))}} 
       \sup_{\bm{\tau} \in \mcl{T}}
            \left \|
                \mI_{n_{l-1}} - \tau_l B_l(\vw^{(j(l))})^\top B_l(\vw^{(j(l))})
            \right \|_{2 \to 2},
    \label{eq:alpha}     \\
    D_\infty & = \max_{l \in [L+1]} D_l, \label{eq:def-Dinft}
\end{align}
and define $Z_0 = 0$, 
\begin{equation}
Z_l  =
\tau_\infty B_\infty \sum_{k=1}^{l}  \alpha^{k} =
\left\{
\begin{array}{ll}
\tau_\infty B_\infty \alpha \frac{1-\alpha^{l}}{1-\alpha} & \mbox{ if } \alpha \neq 1\\
\tau_\infty B_\infty l & \mbox{ if } \alpha = 1\end{array} \right. 
\quad l = 1,\hdots,L.
\label{eq:Z_l}
\end{equation}
Using this, we further define 
\begin{align}
M_L & = \sum_{l=1}^L   
        \left(\lambda_\infty \sqrt{n_{\infty} m}  +  B_\infty \|\mY\|_F  (B_\infty Z_{l-1} + 1) \right)
        \alpha^{L-l},  \label{eq:M_L_threshold} \\
O_L & = \sum_{l=1}^L \tau_\infty \sqrt{n_{\infty} m} 
        \alpha^{L-l}, \label{eq:O_L_threshold} \\
Q_L & =
(B_\infty K_L + \|\mY \|_F Z_L) D_\infty,
\label{eq:Q_L_threshold} 
\end{align}
where $K_L$ in the definition \eqref{eq:Q_L_threshold} of $Q_L$ is given by
\begin{equation}
    K_L = \sum_{l=1}^L \tau_\infty \|\mY\|_F \left( 1 + 2 B_\infty Z_{l-1} \right) 
        \alpha^{L-l}.
        \label{eq:K_L_threshold} 
\end{equation}
We assume that the data distribution $\mathcal{D}$ is such that
for $(\vx,\vy) \sim \mathcal{D}$
\begin{equation}
    \|\vy\|_2 \leq \Bin \quad \mbox{ almost surely }
\label{eq:input_bounded_almost_surely}
\end{equation}
for some constant $\Bin$. In particular, $\|\vy_i\|_2 \leq \Bin$ for all $i =1,\hdots,m$ (with probability $1$).
Furthermore, we require the function
\begin{equation}
\Psi(t) = \sqrt{\log(1+t) + t(\log(1+t)-\log(t))}, \quad t > 0, \quad \Psi(0) = 0.
\label{eq:function_Psi}    
\end{equation}
Note that $\Psi$ is continuous in $0$ and satisfies the bound $\Psi(t) \leq \sqrt{\log(e(1+t))}$, see Lemma~\ref{lem:int-estimate}.
Our main theorem reads as follows.

\begin{theorem}\label{theorem:main_result}
Consider the hypothesis space $\hypspacelong$ defined in \eqref{eq:hypothesis_space_general}.
With probability at least $1-\delta$, the true risk for any 
$h \in \hypspacelong$ is bounded as
\begin{equation}\label{eq:gen-bound-main}
\Ltrue(h)
\leq 
\Lemp(h) +     
2 \sqrt{2} \radem_{\trainS}(\hypspacelong) +
4(\Bin+\Bout)\sqrt{\frac{2\log(4/\delta)}{\dimS}},
\end{equation}
where the Rademacher complexity term 
is further bounded by 
\begin{align}\label{eq:rademacher-bound-main}
\radem_{\trainS}(\hypspacelong) & \leq 
2\sqrt{2}\Bout
\left[\sqrt{\frac{K}{\dimS}}\Psi\left(\frac{16 W_\infty Q_L}{\sqrt{\dimS}\Bout}\right)
+ \sqrt{\frac{L}{\dimS}}\Psi\left(\frac{8r_2 O_L}{\sqrt{\dimS}\Bout}\right)
+ \sqrt{\frac{L}{\dimS}} \Psi\left(\frac{8r_1 M_L}{\sqrt{\dimS} \Bout} \right) \right].
\end{align}
\end{theorem}

The Rademacher complexity is introduced and discussed in more detail in Section \ref{sec:generalization}; in particular see \eqref{rademacher_2}.

\begin{remark}
In the special case that the stepsizes $\bm{\tau_0} \in \R^\dimL$ and/or the thresholds $\bm{\lambda}_0 \in \R^\dimL$ are fixed, so that $r_1=0$ and/or $r_2 = 0$, the above bound simplifies due to $\Psi(0) = 0$. 
For instance, if $r_1=r_2=0$ then
\[
\radem_{\trainS}(\hypspacelong) \leq 
2\sqrt{2}\Bout \sqrt{\frac{K}{\dimS}}\Psi\left(\frac{16 W_\infty Q_L}{\sqrt{\dimS}\Bout}\right).
\]
\end{remark}
While the constants $M_L$, $O_L$ and $Q_L$ look complicated in general and may actually scale exponentially in $L$, the expressions greatly simplify in the important special case that
$\alpha \leq 1$. In fact, the motivating algorithm ISTA, corresponding to fixing $\mA\mPhi$ in \eqref{def:ISTA} and letting $L \to \infty$, is known to converge under the condition that $\tau \|\mA \mPhi\|_{2 \to 2} \leq 1$ implying that $\|\mI - \tau (\mA \mPhi)^\top (\mA\mPhi)\|_{2 \to 2} \leq 1$. These conditions correspond to
$\tau_\infty B_{\infty}^2 \leq 1$ and $\alpha \leq 1$ in our general setup. This suggests to impose these condition on the hypothesis
space (and therefore in the training of the network). 
The corresponding generalization result reads as follows.

\begin{corollary}\label{cor:main} Assume that $\tau_\infty B_\infty^2 \leq 1$, implying $\alpha \leq 1$. Set $n_\infty = \max_{l \in [L]} n_l$. Then the Rademacher complexity term in
\eqref{eq:gen-bound-main} is bounded by
\begin{align}
   \radem_{\trainS}(\hypspacelong) 
   & \leq 2 \sqrt{2} \Bout \left[\sqrt{\frac{K}{\dimS}\log\left(e\left(1+16L(L+1)\tau_\infty B_\infty W_\infty D_\infty \frac{\Bin}{\Bout}\right)\right)} + \right. \label{eq:cor:main}\\
    & \; + \left. \sqrt{\frac{L}{\dimS}}\Psi\left(\frac{8r_2 L \tau_\infty \sqrt{n_{\infty} m} }{\sqrt{\dimS} \Bout}\right) + \sqrt{\frac{L}{\dimS}} \Psi\left(\frac{8r_1L\left( \lambda_\infty^2 n_\infty \sqrt{\dimS} + \Bin(B_\infty \lambda_\infty \sqrt{n_{\infty} m}  + (L-1)/2)\right)}{\Bout}\right)\right]. \notag
\end{align}
\end{corollary}
\begin{proof} Note that $\|\mY\|_F \leq \sqrt{m} \Bin$.
Under our assumptions, the constant $K_L$ satisfies
\begin{align*}
K_L 
&   =    \sum_{l=1}^L \tau_\infty \|\mY\|_F \left( 1 + 2 B_\infty Z_{l-1} \right) \alpha^{L-l} 
   \leq \sum_{l=1}^L \tau_\infty \|\mY\|_F \left( 1 + 2 (l-1) \tau_\infty B_\infty^2  \right) \\
  & \leq  \tau_\infty \|\mY\|_F\left(L +  2 \sum_{l=1}^L (l-1)\right)  
   =  \tau_\infty \|\mY\|_F\left(L+ L(L-1)\right) 
   =  \tau_\infty \|\mY\|_F L^2 
   \leq \tau_\infty L^2 \sqrt{\dimS}\Bin.
\end{align*}
Hence,
\begin{align*}
Q_L &= (B_\infty K_L + \|\mY \|_F Z_L) D_\infty \leq \left[L^2 \tau_\infty B_\infty\sqrt{m}\Bin + \sqrt{\dimS} \Bin L \tau_\infty B_\infty\right] D_\infty \\
&= L(L+1) \tau_\infty B_\infty D_\infty \sqrt{\dimS} \Bin.
\end{align*}
For the constant $O_L$, we obtain
\begin{align*}
        O_L 
=       \sum_{l=1}^L \tau_\infty \sqrt{n_{\infty} m}  \alpha^{L-l}
\leq    
L \tau_\infty \sqrt{n_{\infty} m} .
\end{align*}
The constant $M_L$ satisfies
\begin{align*}
M_L
& =
\sum_{l=1}^L   
        \left(\lambda_\infty \sqrt{n_{\infty} m}   +  B_\infty \|\mY\|_F  (B_\infty Z_{l-1} + 1) \right)
        \alpha^{L-l} \\
& \leq     L (\lambda_\infty  \sqrt{n_{\infty} m}   + B_\infty \|\mY\|_F )
        \lambda_\infty n_{\infty}  + 
        \sum_{l=1}^L   \|\mY\|_F  \tau_\infty B_\infty^2 (l-1)\\ 
         & \leq     L (\lambda_\infty \sqrt{n_{\infty} m}   + B_\infty \sqrt{m} \Bin )
        \lambda_\infty \sqrt{n_{\infty} m}  + 
         \sqrt{\dimS}\Bin  \frac{L(L-1)}{2} . 
\end{align*}
Plugging the above bounds into \eqref{eq:rademacher-bound-main} and using that $\Psi(t) \leq \sqrt{\log(e(1+t))}$ completes the proof.
\end{proof}

\medskip

Note that in case the mappings $B_l$ are linear it follows from $B_\infty \leq W_\infty$, see \eqref{eq:BWD-bound}, that
the assumption $\tau_\infty B_\infty^2 \leq 1$ is implied by $\tau_\infty B_\infty W_\infty D_\infty \leq 1$. Additionally assuming $\Bin=\Bout$, the first logarithmic term in \eqref{eq:cor:main} takes the simple form $\log(e(1+16L(L+1)))$.

In general, considering only the dependence in $K,L$ and $m$ and viewing all other terms as constants, the bound of Corollary \eqref{cor:main} essentially reads as
\[
\radem_{\trainS}(\hypspacelong) \lesssim \sqrt{\frac{(K+L)\log(L)}{m}}.
\]
Moreover, if the thresholds and stepsizes are fixed (not learned), so that $r_1 = r_2 = 0$, we obtain the bound
\[
\radem_{\trainS}(\hypspacelong) \lesssim \sqrt{\frac{K\log(L)}{\dimS}}.
\]
This is one of the main messages of our result: The dependence of the generalization error on the number of layers is at most logarithmic in contrast to many previous results on the generalization error for deep learning, where the scaling in the number of layers is often exponential.

This is one of the main messages of our result: The dependence of the generalization error on the number of layers is only logarithmic in important cases 
in contrast to many previous results on the generalization error for deep learning, where the scaling in the number of layers is often exponential.

\subsection{Examples}\label{sec:examples}


Let us illustrate our scenario and main result with different examples. 

\subsubsection{Learning an orthogonal dictionary}

Let us start with our motivating example from Section~\ref{subsection:motivation_example}.
Here, we choose
$J = 1$ (thus, $j(l) = 1$ for all $l$) and
\begin{align}
\mcl{W}^{(1)} & = \{\mPhi : \mPhi \in O(\dimx)\} \subset \R^{N \times N} \simeq \R^{k_1}, \notag \\
B_l (\mPhi ) & = \mA \mPhi, \quad l=1,\hdots,L, \quad 
B_{L+1}(\mPhi)  = \mPhi, \label{eq:def_Bl}
\end{align}
where all dimensions $n_1 = n_2 = \dots = n_L = N$ are equal, $k_1 = N^2$, and with
$\tau_l = \tau$ and $\lambda_l = \lambda $ being fixed. 
We simply put $P_l = \mI_N$ for all $l \in [L]$.
Let us put $\| \; . \, \|^{(1)} = \| \; . \, \|_{2 \to 2}$, so that
for all $\mPhi \in O(\dimx)$ and $l \in [L]$ we have 
\[
  \|B_l (\mPhi) \|_{2 \to 2} 
= \| \mA \mPhi \|_{2 \to 2}
\leq \|\mA\|_{2 \to 2} \| \mPhi \|^{(1)},
\]
so that $D_l = \|\mA\|_{2 \to 2}$ for all $l \in [L]$ due to the linearity of $B_l$.
Moreover, $\|B_l (\mPhi) \|_{2 \to 2} = \|\mPhi\|_{2 \to 2} = \| \mPhi \|^{(1)}$ resulting in $D_{L+1} = 1$ and $B_\infty = W_\infty = D_\infty = \max\{1, \|\mA\|_{2 \to 2}\}$. 

Assuming that $\tau\max\{1, \|\mA\|_{2 \to 2}\} \leq 1$ and considering that the thresholds and stepsizes are fixed, Corollary~\ref{cor:main} states that
the generalization error bound scales like (with high probability)
\begin{equation}
C \sqrt{\frac{N^2\log(L)}{m}}
%
\label{eq:rough_GE_main_example}
\end{equation}
for a constant $C$ depending on $\tau,\|\mA\|_{2 \to 2},\Bin,\Bout$,
as we have already shown (only for this specific case) in our previous work \cite{behboodi2020generalization} (noting that $N+n \asymp N$).

\subsubsection{Overcomplete dictionaries}
As another important class of dictionaries, we can also consider overcomplete dictionaries.
This case is similar to the previous one, but here we consider
\[
\mcl{W}^{(1)} = 
\{\mPhi : \mPhi \in \R^{N \times p} , \|\mPhi\|_{2 \to 2} \leq \rho \}
\subset \R^{n \times p} \simeq \R^{k_1}, 
\qquad p > N > n, \qquad \rho > 0,
\]
with $k_1 = N \cdot p > n \cdot N$. The mappings $B_l$ are defined
as in \eqref{eq:def_Bl}. We have the input/output
dimensions $n_0 = n_1 = n_2 = \dots = n_{L-1} = p$, and $n_{L} = N$.
We use again $\| \; . \, \|^{(1)} = \| \; . \, \|_{2 \to 2}$,
which, as above, leads to
$D_\infty = B_\infty = W_\infty = \max\{ \rho \|\mA\|_\infty, 1\}$.
%
Assuming constant stepsizes and thresholds and $\tau B_\infty^2 \leq 1$,
Corollary~\ref{cor:main} leads to a generalization bound scaling like
\[
C \sqrt{\frac{Np \log(L)}{\dimS}}.
\]
This is slightly worse than for orthonormal dictionaries due to $N \geq p$.

\subsubsection{Two (alternating) dictionaries}
Similar to the first two examples, one may consider two  
``alternating'' dictionaries. In case of orthogonal dictionaries, similar to the first example, we have 
$J = 2$ (thus, $j(l) = 1$ for all $l$ ) and
\[
\mcl{W}^{(1)} = \{\mPhi_1 : \mPhi_1 \in O(\dimx)\} \subset \R^{N \times N} \simeq \R^{k_1}, 
\qquad
\mcl{W}^{(2)} = \{\mPhi_2 : \mPhi_2 \in O(\dimx)\} \subset \R^{N \times N} \simeq \R^{k_2}. 
\]
The mappings $B_l$ are defined as in \eqref{eq:def_Bl}, but,
for odd $l$, $B_l$ operates on
$\mcl{W}^{(1)}$, while for even $l$ it operates on $\mcl{W}^{(2)}$.
Here, $K = k_1 + k_2 = 2 N^2$, which results in an additional factor of $\sqrt{2}$ appearing in \eqref{eq:rough_GE_main_example}.
Analogously, one may obtain bounds for two alternating overcomplete dictionaries, or more than two alternating dictionaries, or other related scenarios. For example, we can consider  the case without any weight-sharing between layers, i.e., where $J = L$. Then $j(l) = l$, $k_j = N^2$ for all $j = 1, \dots, N$ and $K=k_1 + \dots + k_L = LN^2$. This leads to an additional factor of $\sqrt{L}$. Of course, this may also be combined with trainable stepsizes and thresholds.

\subsubsection{Convolutional LISTA}
For input images, one natural choice of weight matrices is convolutional kernels. In this model, the layer $l$ contains the following operation
\[
    B_l (\vw) (\vz) = \mOmega(\vw) * \vz,
\]
where the length of the convolutional filter $\vw_j$ is $k_j$,  and the mapping $\mOmega : \R^{k_j} \to \R^N$ is its zero-padded version.  This model is discussed in \cite{sreter_learned_2018}. 
Since $K$ is merely dependent on the number of parameters $k_j$ and not $N$, our result already shows that smaller convolutional filters lead to smaller overall $K$ and therefore are expected to show better generalization. We will evaluate convolutional LISTA in the experimental results.


\section{Rademacher Complexity and Dudleys Integral}
\label{sec:generalization}

Let us recall some basic knowledge about the Rademacher complexity and how it is used to bound the generalization error. In order to bound the Rademacher complexity itself, we use the generalized contraction principle from \cite{maurer_vector-contraction_2016} to deal with classes of vector-valued functions. 

\subsection{The Rademacher Complexity}

For a class $\mathcal{G}$ of functions 
$g : Z \to \R$ and a sample $\trainS = (z_1,\hdots,z_m)$
the empirical Rademacher complexity  is defined as 
\begin{equation}
    \radem_\trainS(\mathcal{G}) \defeq \E_{\varepsilon} \sup_{g \in \mathcal{G}} \frac{1}{\dimS} \sum_{i=1}^\dimS \varepsilon_i g(z_i),
\end{equation}
where $\varepsilon$ is a Rademacher vector, i.e., a vector with independent Rademacher variables $\varepsilon_i$, $i=1,\hdots,m$, taking the value $\pm 1$ with equal probability.
The Rademacher complexity is then given as $\radem_\dimS(\mathcal{G})= \E_{\trainS \sim \mathcal{D}^\dimS} \radem_\trainS(\mathcal{G})$, but note that we will exclusively work with the empirical Rademacher complexity. Given a loss function $\ell$ and a hypothesis class $\hypS$, one usually considers the Rademacher complexity of the class $\mathcal{G} = \ell \circ \hypS = \{ g((\vx,\vy)) = \ell(h,\vx,\vy): h \in \hypS\}$.
We rely on the following theorem which bounds the generalization error in terms of the empirical Rademacher complexity.

\begin{theorem}[{\cite[Theorem 26.5]{shalev-shwartz_understanding_2014}}]
Let $\hypS$ be a family of functions, and let $\trainS$ be the training sequence 
drawn from the distribution $\mathcal{D}^\dimS$. Let $\ell$ be a real-valued loss function  satisfying $\card{\ell}\leq c$. Then, for $\delta \in (0,1)$, with probability at least $1-\delta$ 
we have, for all $h \in \hypS$,
\begin{equation}
\Ltrue(h) \leq \Lemp(h) + 2\radem_\trainS(\ell \circ\hypS) + 4c \sqrt{\frac{2\log(4/\delta)}{\dimS}}. 
\label{eq:ge_vs_rademacher}
\end{equation}
\label{thm:ge_vs_rademacher}
\end{theorem} 

To use the above theorem, the loss function needs to be bounded. We make two main assumptions. 
\begin{itemize}
    \item We assume that the input is bounded in the $\ell_2$-norm 
          (see also \eqref{eq:input_bounded_almost_surely} prior to Theorem \ref{theorem:main_result}),
\begin{equation}
\norm{\vy}_2\leq \Bin. 
\label{eq:def_B_in}
\end{equation}
\item We assume that for any $h \in \hypspacelong$
with probability $1$ over $\vy$ chosen from the data distribution
\begin{equation}
\norm{h(\vy)}_2 \leq \Bout.
\label{eq:def_B_out}
\end{equation}
This is ensured if the function $\sigma$ is chosen as in \eqref{eq:function_sigma}, or if a general boundedness assumption for $\sigma$ holds.
\end{itemize}

We will show in Lemma \ref{lemma:bound_output_after_l_layers} in 
Section \ref{sec:bounding_the_output} that in the case of bounded inputs $\vx$ bounded outputs can even be guaranteed without the boundedness assumption on $\sigma$ (and thus, also for the output after hidden layers), although the bound we give in Lemma \ref{lemma:bound_output_after_l_layers} may be improvable in concrete situations.
Under the above assumptions, the loss function $\ell$ (chosen as in \eqref{eq:loss_function} to be the $\ell_2$-distance between the input and the reconstruction) is bounded as
\[
\ell(h,\vy,\vx) = \|h(\vy)-\vx\|_2 \leq  \|\vx\|_2 + \|h(\vy)\|_2 \leq \Bin + \Bout.
\]
Thus, a main challenge and focus in this paper is to bound the 
Rademacher complexity of $\ell \circ \mcl{H}$,
\[
  \radem_\mcl{S} (\ell \circ \mcl{H}) 
= \E \sup_{h \in \mcl{H}} \frac{1}{\dimS} 
    \sum_{i=1}^\dimS \varepsilon_i \norm{\vx_i - h(\vy_i)}_2,
\]
with $\mcl{H} = \hypspacelong$ in the most general case studied here, or some other hypothesis spaces of interest.
Often, e.g., in multiclass classification problems using the margin loss \cite{bartlett_spectrally-normalized_2017}, the function $h$ is real-valued. 
Then, one can apply the classical contraction principle by Talagrand \cite{ledoux_probability_2011} to directly bound the Rademacher complexity of the hypothesis space. However, in our case the function $h$ is vector-valued, and Talagrand's contraction lemma ceases to hold. However, since the norm is 1-Lipschitz, we can use the following generalization of the contraction principle for Rademacher complexities of vector-valued hypothesis classes.

\begin{lemma} [{\cite[Corollary 4]{maurer_vector-contraction_2016}}]\label{lemma_maurer}
Let $\trainS = (\vx_i)_{i\in[\dimS]}$ be a sequence of elements in $\mathcal{X}$. Suppose that the functions $h \in \mathcal{H}$ map $\mathcal{X}$  to $\R^\dimx$, and that $f$ is a $K$-Lipschitz function from $\R^\dimx$ to  $\R$. Then
\begin{equation}
    \E \sup_{h\in\hypS} \sum_{i=1}^\dimS \varepsilon_i f\circ h(\vx_i) \leq \sqrt{2} K \E \sup_{h\in\hypS} \sum_{i=1}^\dimS \sum_{k=1}^\dimx \varepsilon_{ik} h_k(\vx_i).
\end{equation}
\label{lem:Maurer}
\end{lemma}
Note that in our case the $\ell_2$-norm (as any norm) appearing in the loss function is indeed (trivially) $1$-Lipschitz.
Therefore, according to Lemma \ref{lemma_maurer}, it is enough to bound 
the following doubly indexed Rademacher complexity
\begin{equation}
        \radem_\trainS(\ell \circ \mcl{H}) 
\leq    \sqrt{2} \radem_{\trainS}(\mcl{H})
\defeq  \sqrt{2} \E \sup_{h \in \mcl{H}} \frac{1}{\dimS} 
        \sum_{i=1}^\dimS \sum_{k=1}^\dimx \varepsilon_{ik} h_k(\vy_i),
\label{rademacher_2}
\end{equation}
where once again $\mcl{H}$ typically stands for $\hypspacelong$.
In the next sections, we show how \eqref{rademacher_2} can be bounded using Dudleys integral, which itself will be estimated using covering number bounds.

\subsection{Bounding the Rademacher Complexity via Dudley's Integral}

For a fixed number $L \in \N$ of layers, and given a hypothesis space $\mcl{H}$ of functions mapping from
$\R^\dimy$ to $\R^{n_{L+1}}$ (with $n_{L+1} = N$ for reconstruction tasks) let us define the set $\mcl{M}_\mcl{H} \subset \R^{n_{L+1}  \times \dimS}$ as  
\begin{align}
        \mcl{M}_\mcl{H}
\defeq  \left\{ \left(h(\vy_1) | \dots | h(\vy_\dimS) \right) 
                \in \R^{n_{L+1} \times \dimS}: h\in \mcl{H} \right\}.
\end{align}
From now on, we focus on the hypothesis space $\mcl{H} = \hypspacelong$, 
when the corresponding set is given by (using the compact matrix notation)
\begin{align}
\mcl{M}_{\hypspacelong}
=
\left\{ 
h_{\mW, \bm{\tau}, \bm{\lambda}}^L (\mY)
\in \R^{n_{L+1} \times \dimS} 
:
 \mW \in \mcl{W},  \bm{\tau} \in \mcl{T}, \bm{\lambda} \in \Lambda,
\right\} .
\label{eq:hypSet}
\end{align}
In words, $\mcl{M}_{\hypspacelong}$ is the set consisting of all matrices whose columns are the 
outputs of any possible hypothesis applied to the measurements $\vy_i$. In  the compressive sensing scenario,
these are the reconstructions from the measurements in the training set, using any possible decoder in our hypothesis space. If the hypothesis space is clear from the context, we write $\mcl{M}$ instead of $\mcl{M}_{\hypspacelong}$.
In the case $\mcl{H} = \hypspacelong$, the set $\mcl{M}$ is parameterized by 
$\left( \bm{\tau}, \bm{\lambda}, \mW\right) \in 
\mcl{T} \times \Lambda \times \mcl{W}$  
(as $\hypspacelong$ is), so that we can rewrite \eqref{rademacher_2} as
\begin{align}
    \radem_{\trainS}(\hypspacelong)
& = \E \sup_{\mM \in \mcl{M}}\frac{1}{\dimS} \sum_{i=1}^\dimS \sum_{k=1}^{n_L}  \varepsilon_{ik} M_{ik} \\
& = \E 
\sup_{\bm{\tau} \in \mcl{T}} 
\sup_{\bm{\lambda} \in \Lambda} 
\sup_{ \mW  \in \mcl{W}} 
    \frac{1}{\dimS} \sum_{i=1}^\dimS \sum_{k=1}^{n_L}  \varepsilon_{ik} 
    \left(h_{\mW, \bm{\tau}, \bm{\lambda}}^L (\mY) \right)_{ik}.
\label{rademacher_2_rewritten_threshold}
\end{align}
The Rademacher process under consideration has sub-gaussian increments, and therefore, we can apply Dudley's inequality. 
For the set of matrices $\mcl{M}$ defined in \eqref{eq:hypSet}, the radius can be estimated as
\begin{align*}
        \Delta(\mcl M) 
=      \sup_{h\in \hypspacelong} \sqrt{\E\left(\sum_{i=1}^\dimS \sum_{k=1}^{n_L}               
        \varepsilon_{ik} h_k(\vx_i)\right)^2} 
=   \sup_{h\in \hypspacelong} \sqrt{\sum_{i=1}^\dimS \sum_{k=1}^{n_L}  h_k(\vx_i)^2} 
=  \sup_{h\in \hypspacelong} \sqrt{\sum_{i=1}^\dimS  \norm{h(\vx_i)}_2^2} 
\leq   \sqrt{m}\Bout.
\end{align*}

Dudley's inequality, as stated in
\cite[Theorem 8.23.]{foucart_mathematical_2013},
then bounds the Rademacher complexity  as  
\begin{equation}
     \radem_{\trainS}(\mcl{H}) 
\leq \frac{4\sqrt{2}}{\dimS} \int_0^{\sqrt{\dimS}\Bout/2}
        \sqrt{\log\cover(\mcl{M}, \|\cdot\|_F, \varepsilon)} \d \varepsilon.
\label{eq:dudley_bound}
\end{equation}

To derive the generalization bound, it suffices to bound the covering numbers of $\mcl{M}$. After some preparation, this will be done in Section~\ref{subsec:covering_nb_est_thresholds}.

\subsection{Bounding the output}
\label{sec:bounding_the_output}

As a first auxiliary tool, we prove a bound for the output of the network, after any number of (possibly intermediate) layers $l$, in the next lemma. 
We state a general version which allows possibly different stepsizes and thresholds for each layer. 

\begin{lemma}\label{lemma:bound_output_after_l_layers}
For $l=1,\hdots,L$ and
$\mW = \left(\vw^{(1)}, \dots \vw^{(J)}\right)  \in \mathcal{W}$,  
$\bm{\tau} = (\tau_1, \dots, \tau_l )  \in \mathcal{T}$ and
$\bm{\lambda} = (\lambda_1, \dots, \lambda_l ) \in \Lambda$,  
we have
\begin{align}
     &     \left \| f_{\mW, \bm{\tau}, \bm{\lambda}}^l (\mY) \right \|_F \nonumber \\
\leq & 
\sum_{k=1}^{l}  
\left(
    \left \| \tau_k  B_{k}(\vw^{(j(k))})^\top \mY  \right \|_F
    \prod_{i = k}^{l - 1}
    \left \| \mI_{n_{i}} -  \tau_i B_{i+1}(\vw^{(j(i+1))})^\top  B_{i+1}(\vw^{(j(i+1))}) \right \|_{2\to 2}
\right) 
\label{eq:norm_bound_individual_1_threshold} \\
\leq & 
 \| \mY  \|_F \sum_{k=1}^{l}  
\left(
\tau_k
    \left \|  B_{k}(\vw^{(j(k))}) \right \|_{2 \to 2}
    \prod_{i = k}^{l-1}
    \left \| \mI_{n_{i}} -  \tau_i B_{i+1}(\vw^{(j(i+1))})^\top  B_{i+1}(\vw^{(j(i+1))}) \right \|_{2\to 2}
\right),
\label{eq:norm_bound_individual_2_threshold}
\end{align}
following the usual convention of defining the empty product as one.
\end{lemma}

\begin{proof}
We just prove the first inequality \eqref{eq:norm_bound_individual_1_threshold}, 
as the second inequality \eqref{eq:norm_bound_individual_2_threshold} follows then immediately.
We proceed via induction. For $l = 1$, using the 
norm contractivity
of $P_l$ and of the soft thresholding operator we obtain 
\[
    \left \| f_{\mW, \bm{\tau}, \bm{\lambda}}^1 (\mY) \right \|_F
=    \left \|P_{1} S_{\tau_1, \lambda_1} \left(\tau B_1(\vw^{(j(1))}) ^\top \mY  \right) \right \|_F
\leq \left \|\tau_1 B_1(\vw^{(j(1))}) ^\top  \mY  \right \|_F.
\]
Assuming the statement is true for $l$, 
we obtain 
\begin{align*}
  & \left \| f_{\mW, \bm{\tau}, \bm{\lambda}}^{l+1} (\mY) \right \|_F  \\
\leq  & 
\left \| \mI_{n_{l}} - \tau_{l+1} B_{l+1}(\vw^{(j(l+1))})^\top  B_{l+1}(\vw^{(j(l+1))}) \right \|_{2 \to 2}
\left \| f_{\mW, \bm{\tau}, \bm{\lambda}}^l (\mY) \right \|_F
     + \left \| \tau_{l+1}  B_{l+1}(\vw^{(j(l+1))})^\top \mY \right \|_F \\
\leq  & 
\sum_{k=1}^{l}  
\left(
    \left \| \tau_k  B_{k}(\vw^{(j(k))})^\top \mY  \right \|_F
    \prod_{i = k}^{l}
    \left \| \mI_{n_{i}} -  \tau_{i+1} B_{i+1}(\vw^{(j(i+1))})^\top  B_{i+1}(\vw^{(j(i+1))}) \right \|_{2\to 2}
\right) \\
& \qquad + \left \| \tau_{l+1}  B_{l+1}(\vw^{(j(l+1))})^\top \mY \right \|_F \\ 
\leq  & 
\sum_{k=1}^{l+1}  
\left(
    \left \| \tau_k  B_{k}(\vw^{(j(k))})^\top \mY  \right \|_F
    \prod_{i = k}^{l}
    \left \| \mI_{n_{i}} -  \tau_{i+1} B_{i+1}(\vw^{(j(i+1))})^\top  B_{i+1}(\vw^{(j(i+1))}) \right \|_{2\to 2}
\right).
\qedhere
\end{align*}
This is the claimed inequality for $l+1$ and completes the induction.
\end{proof}

We immediately obtain the following corollary
that bounds the output of the full network. 

\begin{corollary}
For any $h = h_{\mW, \bm{\tau}, \bm{\lambda}} \in \hypspacelong$, the output is bounded by
\begin{equation}
    \left \| h(\mY) \right \|_F  
=   \left\| \sigma \left( B_{L+1}(\vw^{(j(L+1))}) f_{\mW, \bm{\tau}, \bm{\lambda}}^{L} (\mY) \right) \right \|_F
\leq 
 B_\infty       
        \left\|f_{\mW, \bm{\tau}, \bm{\lambda}}^{L} (\mY) \right \|_F.
\end{equation}
\end{corollary}
\begin{proof} 
The statement follows immediately from the fact that $\sigma$ from \eqref{eq:function_sigma}
is norm-contractive \eqref{ABHRES:sigma:bound}.
\end{proof}

\section{Proofs}
\label{sec:proof}

In this section we prove our main result Theorem~\ref{theorem:main_result}. 
We consider the most general scenario introduced in Section \ref{subsec:general_setup} with $\mcl{H} = \hypspacelong$ defined in 
\eqref{eq:hypothesis_space_general}.
The main ingredient for bounding the covering numbers of $\mcl{M}$, as is required to continue from \eqref{eq:dudley_bound} on,
will be Lipschitz estimates of the neural networks with respect to the parameters, i.e., bounds for (again using the compact matrix notation)
\begin{equation}
    \left\| 
        f_{\bm{\tau}^{(1)}, \bm{\lambda}^{(1)}, \mW_1}^L (\mY) - 
        f_{\bm{\tau}^{(2)}, \bm{\lambda}^{(2)}, \mW_2}^L (\mY) 
    \right\|_F,
    \label{eq:bound_this_thresholds}
\end{equation}
with respect to the differences of the individual involved parameters 
(for $l=1, \dots, L$)
\begin{equation}
\left| \lambda^{(2)}_l-\lambda^{(1)}_l \right|,
\qquad
\left| \tau^{(2)}_l-\tau^{(1)}_l \right| ,
\qquad
\left\|  B_l(\vw_1^{(j(l))}) - B_l(\vw_2^{(j(l))}) \right\|_{2 \to 2}.
\label{eq:differences_parameters}
\end{equation}
Here $\bm{\tau}^{(i)}$, $\bm{\lambda}^{(i)}$ and $\mW_i$ denote the different stepsizes, thresholds and parameters for the $B_l$ functions for $i=1,2$.
To shorten the notation in the following, we will summarize the respective parameters in a vector $\mcl{P}$ and
write $f_{\mcl{P}}^{L}(\mY)$ and $h_{\mcl{P}}^{L}(\mY)$.

Let us note that the upper bounds in \eqref{eq:norm_bound_individual_1_threshold}
and \eqref{eq:norm_bound_individual_2_threshold} do not depend on the threshold $\lambda_l$. 
However, this is not the case anymore when it comes to the perturbation bound. It is easy to verify that
$\left| S_{\tau_2\lambda_2}(x)-S_{\tau_1\lambda_1}(x)\right| \leq \left| \tau_2\lambda_2-\tau_1\lambda_1\right|$ 
for arbitrary $x \in \R$ and $\tau_1, \tau_2, \lambda_1, \lambda_2 > 0$. This implies that, for a vector $\vx \in \R^N$, we have
\begin{equation}
\left\|    S_{\tau_2\lambda_2}(\vx)-S_{\tau_1\lambda_1}(\vx)\right\|_2 
\leq 
\sqrt{N} \left| \tau_2\lambda_2-\tau_1\lambda_1\right|
\label{eq:different_thresholds_vector}
\end{equation}
and similarly, for a matrix $\mX \in \R^{N \times m}$,
\begin{equation}
\left\|    S_{\tau_2\lambda_2}(\mX)-S_{\tau_1\lambda_1}(\mX)\right\|_F 
\leq 
\sqrt{m N} \left| \tau_2\lambda_2-\tau_1\lambda_1\right|.
\label{eq:different_thresholds_matrix}
\end{equation}

To simplify the notation further, let us introduce the following quantities
\begin{align}
\xi_l 
& := 
\left | \tau^{(2)}_l\lambda^{(2)}_l-\tau^{(1)}_l\lambda^{(1)}_l \right| 
\leq 
\tau_\infty \left| \lambda^{(2)}_l-\lambda^{(1)}_l \right| + 
\lambda_\infty  \left| \tau^{(2)}_l-\tau^{(1)}_l \right| 
\label{eq:def_xi_threshold}\\
\delta_l 
& := 
\left \|\tau^{(1)}_l   B_l(\vw_1^{(j(l))}) -\tau^{(2)}_l B_l(\vw_2^{(j(l))}) \right\|_{2 \to 2} \label{eq:def_delta_threshold}\\
&\leq B_\infty \left | \tau^{(1)}_l  - \tau^{(2)}_l  \right |
       +
        \tau_\infty 
        \left \| 
        B_l(\vw_2^{(j(l))}) - \tau^{(2)}_l B_l(\vw_1^{(j(l))}) 
        \right \|_{2 \to 2}
\label{eq:delta_bound}
\\
\gamma_l 
& :=
\left \|
        \left(
            \mI_{n_{l-1}} -\tau^{(1)}_l B_l(\vw_1^{(j(l))})^\top  B_l(\vw_1^{(j(l))}) 
        \right ) f^{l-1}_{\mcl{P}_1}(\mY) \right. 
\label{eq:def_gamma_threshold} \\
    & \left. \qquad        -
        \left(
            \mI_{n_{l-1}} -  \tau^{(2)}_l B_l(\vw_2^{(j(l))})^\top  B_l(\vw_2^{(j(l))}) 
        \right ) f^{l-1}_{\mcl{P}_2}(\mY)  
    \right\|_F .
\nonumber
\end{align}
The given estimates for $\xi_l$  and $\delta_l$
follow immediately from the triangle inequality and the definition 
of $\tau_\infty$, $\lambda_\infty$ and $B_\infty$ in \eqref{eq:standing_assumptions}.
%
%
%
%
The following lemma provides a useful bound also for the quantity $\gamma_l$.

\begin{lemma} \label{lem:simple_estimates}
It holds
\begin{align*}
        \gamma_l
        & \leq
         2 \tau_\infty B_\infty \left \| f^{l-1}_{\mcl{P}_1}(\mY)  \right \|_F
    \left \|
              B_l(\vw_2^{(j(l))})  - B_l(\vw_1^{(j(l))}) 
    \right\|_{2 \to 2} \\
   & + 
    B_\infty^2
    \left \| f^{l-1}_{\mcl{P}_1}(\mY)  \right \|_F 
    \left | \tau^{(1)}_l  - \tau^{(2)}_l  \right |  + 
       \alpha \left \|
        f^{l-1}_{\mcl{P}_2}(\mY)  - f^{l-1}_{\mcl{P}_1}(\mY)  
        \right\|_F . 
\end{align*}
\end{lemma}


\begin{proof}
We obtain
\begin{align*}
    & \left \|
        \left(
            \mI_{n_{l-1}} -\tau^{(1)}_l B_l(\vw_1^{(j(l))})^\top  B_l(\vw_1^{(j(l))}) 
        \right ) f^{l-1}_{\mcl{P}_1}(\mY)  
            -
        \left(
            \mI_{n_{l-1}} -  \tau^{(2)}_l B_l(\vw_2^{(j(l))})^\top  B_l(\vw_2^{(j(l))}) 
        \right ) f^{l-1}_{\mcl{P}_2}(\mY)  
    \right\|_F \\
\leq &
    \left \|
        \left(
            \mI_{n_{l-1}} -\tau^{(1)}_l B_l(\vw_1^{(j(l))})^\top  B_l(\vw_1^{(j(l))}) 
        \right ) f^{l-1}_{\mcl{P}_1}(\mY)  
            -
        \left(
            \mI_{n_{l-1}} -  \tau^{(1)}_l B_l(\vw_1^{(j(l))})^\top  B_l(\vw_2^{(j(l))}) 
        \right ) f^{l-1}_{\mcl{P}_1}(\mY)  
    \right. \\
   & + \left .
        \left(
            \mI_{n_{l-1}} -\tau^{(1)}_l B_l(\vw_1^{(j(l))})^\top  B_l(\vw_2^{(j(l))}) 
        \right ) f^{l-1}_{\mcl{P}_1}(\mY)  
            -
        \left(
            \mI_{n_{l-1}} -  \tau^{(2)}_l B_l(\vw_2^{(j(l))})^\top  B_l(\vw_2^{(j(l))}) 
        \right ) f^{l-1}_{\mcl{P}_1}(\mY)  
    \right. \\
   & + \left.
        \left(
            \mI_{n_{l-1}} -\tau^{(2)}_l B_l(\vw_2^{(j(l))})^\top  B_l(\vw_2^{(j(l))}) 
        \right ) f^{l-1}_{\mcl{P}_1}(\mY)  
            -
        \left(
            \mI_{n_{l-1}} -  \tau^{(2)}_l B_l(\vw_2^{(j(l))})^\top  B_l(\vw_2^{(j(l))}) 
        \right ) f^{l-1}_{\mcl{P}_2}(\mY)  
    \right\|_F \\
\leq &
    \left \|
        \left(
              \tau^{(1)}_l B_l(\vw_1^{(j(l))})^\top  B_l(\vw_2^{(j(l))})  
            - \tau^{(1)}_l B_l(\vw_1^{(j(l))})^\top  B_l(\vw_1^{(j(l))}) 
        \right ) f^{l-1}_{\mcl{P}_1}(\mY)  
    \right\|_F \\
   & + \left \|
        \left(
            \tau^{(2)}_l B_l(\vw_2^{(j(l))})^\top  B_l(\vw_2^{(j(l))})  -
            \tau^{(1)}_l B_l(\vw_1^{(j(l))})^\top  B_l(\vw_2^{(j(l))}) 
        \right ) f^{l-1}_{\mcl{P}_1}(\mY)  
    \right \|_F \\
   & + \left \|
        \left(
            \mI_{n_{l-1}} -  \tau^{(2)}_l B_l(\vw_2^{(j(l))})^\top  B_l(\vw_2^{(j(l))}) 
        \right ) 
        \left( 
        f^{l-1}_{\mcl{P}_2}(\mY)  - f^{l-1}_{\mcl{P}_1}(\mY)  
        \right)
    \right\|_F \\
\leq &
    \tau_\infty B_\infty \left \| f^{l-1}_{\mcl{P}_1}(\mY)  \right \|_F
    \left \|
              B_l(\vw_2^{(j(l))})  - B_l(\vw_1^{(j(l))}) 
    \right\|_{2 \to 2} 
    + \delta_l  B_\infty   
    \left \| f^{l-1}_{\mcl{P}_1}(\mY)  \right \|_F \\
   & + \left \|
            \mI_{n_{l-1}} -  \tau^{(2)}_l B_l(\vw_2^{(j(l))})^\top  B_l(\vw_2^{(j(l))}) 
        \right \|_{2 \to 2}
        \left \|
        f^{l-1}_{\mcl{P}_2}(\mY)  - f^{l-1}_{\mcl{P}_1}(\mY)  
        \right\|_F \\
\leq & \tau_\infty B_\infty \left \| f^{l-1}_{\mcl{P}_1}(\mY)  \right \|_F
    \left \|
              B_l(\vw_2^{(j(l))})  - B_l(\vw_1^{(j(l))}) 
    \right\|_{2 \to 2}  
         + \alpha \left \|
        f^{l-1}_{\mcl{P}_2}(\mY)  - f^{l-1}_{\mcl{P}_1}(\mY)  
        \right\|_F\\
      &  + \left(B_\infty \left | 
    \tau^{(1)}_l  - \tau^{(2)}_l  \right | 
       + \tau_\infty 
        \left \| 
        B_l(\vw_2^{(j(l))}) - \tau^{(2)}_l B_l(\vw_1^{(j(l))}) 
        \right \|_{2 \to 2}\right) B_\infty    
    \left \| f^{l-1}_{\mcl{P}_1}(\mY)  \right \|_F\\
    & = 2 \tau_\infty B_\infty \left \| f^{l-1}_{\mcl{P}_1}(\mY)  \right \|_F
    \left \|
              B_l(\vw_2^{(j(l))})  - B_l(\vw_1^{(j(l))}) 
    \right\|_{2 \to 2}  
         + \alpha \left \|
        f^{l-1}_{\mcl{P}_2}(\mY)  - f^{l-1}_{\mcl{P}_1}(\mY)  
        \right\|_F\\
        &+ B_\infty^2 \left \| f^{l-1}_{\mcl{P}_1}(\mY)  \right \|_F  \left | 
    \tau^{(1)}_l  - \tau^{(2)}_l  \right |. 
\end{align*}
Hereby, we have used the estimate \eqref{eq:delta_bound} for $\delta_l$ and 
that $\left \|\mI_{n_{l-1}} -  \tau^{(2)}_l B_l(\vw_2^{(j(l))})^\top  B_l(\vw_2^{(j(l))}) 
        \right \|_{2 \to 2} \leq \alpha$ by definition \eqref{eq:alpha} of $\alpha$.
\end{proof}

Next we state our main technical result, which will be a key ingredient for the covering number estimate, and thus for deriving the generalization bounds. It bounds the perturbation of the output of a network with respect to changes in the parameters.

\begin{theorem}\label{thm:perturbation_threshold}
Consider the functions $f_{\bm{\tau}, \bm{\lambda}, \mW}$ as defined in 
\eqref{eq:def_f^L} with $L \geq 2$. 
Then, for any two such functions parameterized by 
$\left( \bm{\tau}^{(1)}, \bm{\lambda}^{(1)}, \mW_1 \right), 
\left( \bm{\tau}^{(2)}, \bm{\lambda}^{(2)}, \mW_2 \right) 
\in \mcl{T} \times \Lambda \times \mcl{W}$ 
we have
\begin{align}
     & \left\| f_{\bm{\tau}^{(1)}, \bm{\lambda}^{(1)}, \mW_1}^L (\mY) - 
          f_{\bm{\tau}^{(2)}, \bm{\lambda}^{(2)}, \mW_2}^L (\mY) \right\|_F  \nonumber\\
\leq &
K_L  \cdot 
\max_{l \in [L]} 
\left \| B_l(\vw_1^{(j(l))}) - B_l(\vw_2^{(j(l))}) \right\|_{2 \to 2} 
+ M_L  \cdot \left\|  \bm{\tau}^{(1)} -  \bm{\tau}^{(2)} \right \|_\infty 
+ O_L  \cdot \left\|  \bm{\lambda}^{(1)} - \bm{\lambda}^{(2)} \right \|_\infty,\label{eq:estimate_via_K_L_threshold} 
\end{align}
with $K_L$, $O_L$ and $M_L$ all being defined before Theorem \ref{theorem:main_result} in \eqref{eq:K_L_threshold}, \eqref{eq:O_L_threshold} and \eqref{eq:M_L_threshold}.
\end{theorem}

\begin{proof}
For the sake of avoiding to treat the case $l=1$ separately, we formally
introduce the notation $f^{0}_{\mcl{P}_1}(\mY) = f^{0}_{\mcl{P}_2}(\mY) = \mY$.
As a first step, using that 
$P_l$ is $1$-Lipschitz in the first inequality and 
basic properties of the involved norms in the second inequality, 
and applying \eqref{eq:different_thresholds_matrix} for the third inequality,
we obtain 
\begin{align}
&\left\|f^{l}_{\mcl{P}_1}(\mY)-f^{l}_{\mcl{P}_2}(\mY) \right\|_F \label{eq:estimate_perturbation_no_Psi}\\
=    & \left \|
                P_l S_{\tau^{(1)}_l\lambda^{(1)}_l} \left[ \left(\mI_{n_{l-1}} -\tau^{(1)}_l B_l(\vw_1^{(j(l))})^\top  B_l(\vw_1^{(j(l))}) 
                \right )f^{l-1}_{\mcl{P}_1}(\mY) + \tau^{(1)}_l  B_l(\vw_1^{(j(l))})^\top  \mY \right] \right. \nonumber \\
     & \quad - \left.
               P_l S_{\tau^{(2)}_l\lambda^{(2)}_l} \left[ \left(\mI_{n_{l-1}} -  \tau^{(2)}_l B_l(\vw_2^{(j(l))})^\top  B_l(\vw_2^{(j(l))}) 
               \right ) f^{l-1}_{\mcl{P}_2}(\mY) + \tau^{(2)}_l  B_l(\vw_2^{(j(l))})^\top  \mY \right]
        \right\|_F \nonumber \displaybreak[2] \\
\leq    & \left \|
                S_{\tau^{(1)}_l\lambda^{(1)}_l} \left[ \left(\mI_{n_{l-1}} -\tau^{(1)}_l B_l(\vw_1^{(j(l))})^\top  B_l(\vw_1^{(j(l))}) 
                \right ) f^{l-1}_{\mW_1} (\mY) + \tau^{(1)}_l  B_l(\vw_1^{(j(l))})^\top  \mY \right] \right. \nonumber \\
     & \quad - \left.
               S_{\tau^{(2)}_l\lambda^{(2)}_l} \left[ \left(\mI_{n_{l-1}} -  \tau^{(2)}_l B_l(\vw_2^{(j(l))})^\top  B_l(\vw_2^{(j(l))}) 
               \right ) f^{l-1}_{\mcl{P}_2}(\mY)+ \tau^{(2)}_l  B_l(\vw_2^{(j(l))})^\top  \mY \right]
        \right\|_F \displaybreak[2]\nonumber  \\
\leq    & \left \|
                S_{\tau^{(1)}_l\lambda^{(1)}_l} \left[ \left(\mI_{n_{l-1}} -\tau^{(1)}_l B_l(\vw_1^{(j(l))})^\top  B_l(\vw_1^{(j(l))}) 
                \right ) f^{l-1}_{\mcl{P}_1}(\mY) + \tau^{(1)}_l  B_l(\vw_1^{(j(l))})^\top  \mY \right] \right. \nonumber \\
     & \quad - \left.
               S_{\tau^{(2)}_l\lambda^{(2)}_l} \left[ \left(\mI_{n_{l-1}} -  \tau^{(1)}_l B_l(\vw_1^{(j(l))})^\top  B_l(\vw_1^{(j(l))}) 
               \right ) f^{l-1}_{\mcl{P}_1}(\mY) + \tau^{(1)}_l  B_l(\vw_1^{(j(l))})^\top  \mY \right]
        \right\|_F \nonumber  \\
    & \quad + \left \|
                S_{\tau^{(2)}_l\lambda^{(2)}_l} \left[ \left(\mI_{n_{l-1}} -\tau^{(1)}_l B_l(\vw_1^{(j(l))})^\top  B_l(\vw_1^{(j(l))}) 
                \right ) f^{l-1}_{\mcl{P}_1}(\mY) + \tau^{(1)}_l  B_l(\vw_1^{(j(l))})^\top  \mY \right] \right. \nonumber \\
     & \quad - \left.
               S_{\tau^{(2)}_l\lambda^{(2)}_l} \left[ \left(\mI_{n_{l-1}} -  \tau^{(2)}_l B_l(\vw_2^{(j(l))})^\top  B_l(\vw_2^{(j(l))}) 
               \right ) f^{l-1}_{\mcl{P}_2}(\mY) + \tau^{(2)}_l  B_l(\vw_2^{(j(l))})^\top  \mY \right]
        \right\|_F \displaybreak[2]\nonumber  \\
\leq & \left| \tau^{(2)}_l\lambda^{(2)}_l-\tau^{(1)}_l\lambda^{(1)}_l \right| 
\sqrt{n_{l-1} m} \nonumber \\
     & \quad + \left \|
                 \left(\mI_{n_{l-1}} -\tau^{(1)}_l B_l(\vw_1^{(j(l))})^\top  B_l(\vw_1^{(j(l))}) 
                \right ) f^{l-1}_{\mcl{P}_1}(\mY)  + \tau^{(1)}_l  B_l(\vw_1^{(j(l))})^\top  \mY  \right. \nonumber  \\
    & \quad - \left.
                \left(\mI_{n_{l-1}} -  \tau^{(2)}_l B_l(\vw_2^{(j(l))})^\top  B_l(\vw_2^{(j(l))}) 
               \right ) f^{l-1}_{\mcl{P}_2}(\mY)  - \tau^{(2)}_l  B_l(\vw_2^{(j(l))})^\top  \mY  
        \right\|_F \displaybreak[2]\nonumber  \\
\leq & \left| \tau^{(2)}_l\lambda^{(2)}_l-\tau^{(1)}_l\lambda^{(1)}_l \right| 
\sqrt{n_{l-1} m} \nonumber \\
     & \quad + \left \|
                 \left(\mI_{n_{l-1}} -\tau^{(1)}_l B_l(\vw_1^{(j(l))})^\top  B_l(\vw_1^{(j(l))}) 
                \right ) f^{l-1}_{\mcl{P}_1}(\mY)  
                -
               \left(\mI_{n_{l-1}} -  \tau^{(2)}_l B_l(\vw_2^{(j(l))})^\top  B_l(\vw_2^{(j(l))}) 
               \right ) f^{l-1}_{\mcl{P}_2}(\mY)  
                \right\|_F \nonumber  \\
    & \quad + \left\|
                    \tau^{(1)}_l  B_l(\vw_1^{(j(l))})^\top  \mY 
               -    \tau^{(2)}_l  B_l(\vw_2^{(j(l))})^\top  \mY  
        \right\|_F \nonumber  \\
\leq & \xi_l \sqrt{n_{l-1} m}  + \gamma_l + \delta_l \|\mY \|_F,
\end{align}
%
%
using the abbreviations introduced in    
\eqref{eq:def_xi_threshold},
\eqref{eq:def_delta_threshold}
and
\eqref{eq:def_gamma_threshold}.
Inserting the estimates for $\xi_l$, $\delta_l$ and $\gamma_l$ in
\eqref{eq:def_xi_threshold}, \eqref{eq:delta_bound} and Lemma~\ref{lem:simple_estimates}, and using 
$\sqrt{n_{l-1} m} \leq \sqrt{n_{\infty} m}$, we obtain
\begin{align*}
& \left\|f^{l}_{\mcl{P}_1}(\mY)-f^{l}_{\mcl{P}_2}(\mY) \right\|_F 
\leq   \sqrt{n_{\infty} m}   \tau_\infty \left| \lambda^{(2)}_l-\lambda^{(1)}_l \right| + 
       \sqrt{n_{\infty} m} \lambda_\infty  \left| \tau^{(2)}_l-\tau^{(1)}_l \right| 
       \nonumber \\
     & \quad       
     + 2\tau_\infty B_\infty \left \| f^{l-1}_{\mcl{P}_1}(\mY)  \right \|_F
    \left \|
              B_l(\vw_2^{(j(l))})  - B_l(\vw_1^{(j(l))}) 
    \right\|_{2 \to 2} 
    + B_\infty^2
    \left \| f^{l-1}_{\mcl{P}_1}(\mY)  \right \|_F 
    \left | \tau^{(1)}_l  - \tau^{(2)}_l  \right | \displaybreak[2]\\
   & \quad + 
     \alpha
        \left \|
        f^{l-1}_{\mcl{P}_2}(\mY)  - f^{l-1}_{\mcl{P}_1}(\mY)  
        \right\|_F 
    +
     B_\infty \|\mY \|_F  \left | \tau^{(1)}_l  - \tau^{(2)}_l  \right |
       +
        \tau_\infty \|\mY \|_F 
        \left \| 
        B_l(\vw_2^{(j(l))}) - \tau^{(2)}_l B_l(\vw_1^{(j(l))}) 
        \right \|_{2 \to 2}. \nonumber \displaybreak[2]\\
\end{align*}
Recall that by Lemma~\ref{lemma:bound_output_after_l_layers} we have, for $\ell = 1,\hdots,L$,
\begin{align}
\left \| f^{l}_{\mcl{P}_1}(\mY)  \right \|_F & \leq  
\|\mY \|_F \sum_{k=1}^l 
\left(
\tau_k \left \|  B_{k}(\vw_1^{(j(k))}) \right \|_{2 \to 2}
\prod_{i = k}^{l-1}
\left \| \mI_{n_{i}} -\tau_i B_{i+1}(\vw_1^{(j(i+1))})^\top  B_{i+1}(\vw_1^{(j(i+1))})\right\|_{2\to 2}
\right) \notag \\
& \leq \|\mY\|_F  
\tau_\infty B_\infty \sum_{k=1}^{l}  
    \left(
    \sup_{i = k, \dots, l-1} \left \| \mI_{n_{i}} - \tau_i B_{i+1}(\vw_2^{(j(i+1))})^\top  B_{i+1}(\vw_2^{(j(i+1))})\right\|_{2\to 2}
    \right)^{l-k} \notag\\
& \leq  \|\mY\|_F \tau_\infty B_\infty \sum_{k=1}^{l} \alpha^{l-k} 
=
 \|\mY\|_F Z_l,
 \label{eq:Zl-bound2}
\end{align}
with $\alpha$ as defined in \eqref{eq:alpha} and $Z_l$ as in \eqref{eq:Z_l}. 
This leads to the estimate
\begin{align*}
   &  \left\|f^{l}_{\mcl{P}_1}(\mY)-f^{l}_{\mcl{P}_2}(\mY) \right\|_F \\ 
& \leq      
        \alpha
        \left \|
            f^{l-1}_{\mcl{P}_2}(\mY)  - f^{l-1}_{\mcl{P}_1}(\mY)  
        \right\|_F  
         + \tau_\infty \|\mY\|_F \left( 1  + 2 B_\infty  Z_{l-1} \right)
        \left \| B_l(\vw_2^{(j(l))})  - B_l(\vw_1^{(j(l))}) \right\|_{2\to 2}\\
        & \quad  +
        \left(\lambda_\infty \sqrt{n_{\infty} m}  +  B_\infty \|\mY\|_F  (B_\infty Z_{l-1} + 1) \right) 
        \left|\tau^{(1)}_l   -\tau^{(2)}_l\right| 
       +  \tau_\infty \sqrt{n_{\infty} m}  \left| \lambda^{(2)}_l-\lambda^{(1)}_l \right|.
\end{align*}
%
%
%
Introducing the additional quantities
\begin{align*}
\beta_l     &   =  
\tau_\infty \|\mY\|_F \left( 1  + 2B_\infty  Z_{l-1} \right), \\
\kappa_l    &  = 
\left(
\lambda_\infty \sqrt{n_{\infty} m} +  B_\infty \|\mY\|_F  (B_\infty Z_{l-1} + 1) 
\right), \\
\varphi_l & = \tau_\infty \sqrt{n_{\infty} m},
\end{align*}
we can write our estimate more compactly as
\begin{align}
        &  \left\|f^{l}_{\mcl{P}_1}(\mY)-f^{l}_{\mcl{P}_2}(\mY) \right\|_F 
        \nonumber\\
\leq    &  
        \alpha
        \left\|f^{l-1}_{\mcl{P}_1}(\mY)-f^{l-1}_{\mcl{P}_2}(\mY) \right\|_F 
        +
        \beta_l
        \left \| B_l(\vw_2^{(j(l))})  - B_l(\vw_1^{(j(l))}) \right\|_{2\to 2} 
        \nonumber \\
       & \quad +
        \kappa_l
        \left|\tau^{(1)}_l   -\tau^{(2)}_l\right| 
         +
        \varphi_l \left| \lambda^{(2)}_l-\lambda^{(1)}_l \right|, \label{eq:uga}
\end{align}
Using our abbreviations, the general formulas for $K_L$, $M_L$ and $O_L$ for $L \geq 1$ are given by
\begin{equation}
K_L  = \sum_{l=1}^L \beta_l \alpha^{L-l}, 
\qquad
M_L  = \sum_{l=1}^L \kappa_l \alpha^{L-l}, 
\qquad
O_L  = \sum_{l=1}^L \varphi_l \alpha^{L-l},
\qquad L \geq 1,
\label{eq:K_L_M_L_O_L_in_proof}
\end{equation}
which indeed for $L \geq 2$ is just a compact notation for \eqref{eq:K_L_threshold},  \eqref{eq:M_L_threshold} and \eqref{eq:O_L_threshold} in Theorems 
\ref{theorem:main_result} and \ref{thm:perturbation_threshold}.
We now prove via induction that \eqref{eq:estimate_via_K_L_threshold} holds for any number of layers $L \in \N$ with $K_L$, $M_L$ and $O_L$ as just stated.
For $L=1$, we can directly obtain these factors from the following estimate.
Using similar arguments as above, we obtain (with $n_\infty = n_0$)
\begin{align*}
     & \left\|f^{1}_{\mcl{P}_1}(\mY)-f^{1}_{\mcl{P}_2}(\mY) \right\|_F \\
=    & \left \|
            P_1 S_{ \tau_1^{(1)}  \lambda^{(1)}_1} \left[ \tau_1^{(1)}  B_1(\vw_1^{(j(1))})^\top  \mY \right]  -
            P_1 S_{ \tau_1^{(2)}  \lambda^{(2)}_1} \left[ \tau_1^{(2)}   B_1(\vw_2^{(j(1))})^\top  \mY \right]
        \right\|_F \displaybreak[2]\\
\leq  & \left \|
            S_{ \tau_1^{(1)}  \lambda^{(1)}_1} \left[ \tau_1^{(1)}  B_1(\vw_1^{(j(1))})^\top  \mY \right]  -
            S_{ \tau_1^{(2)}  \lambda^{(2)}_1} \left[ \tau_1^{(2)}  B_1(\vw_2^{(j(1))})^\top \mY \right]
        \right\|_F \displaybreak[2]\\
\leq  & \left \|
            S_{ \tau_1^{(1)}  \lambda^{(1)}_1} \left[ \tau_1^{(1)}  B_1(\vw_1^{(j(1))})^\top  \mY \right]  -
            S_{ \tau_1^{(1)}  \lambda^{(1)}_1} \left[ \tau_1^{(2)}  B_1(\vw_1^{(j(1))})^\top  \mY \right]        \right\|_F \\ 
     & \qquad + \left \|
            S_{ \tau_1^{(1)}  \lambda^{(1)}_1} \left[ \tau_1^{(2)}  B_1(\vw_1^{(j(1))})^\top  \mY \right]  -
            S_{ \tau_1^{(2)}  \lambda^{(2)}_1} \left[ \tau_1^{(2)}  B_1(\vw_2^{(j(1))})^\top \mY \right]
        \right\|_F \displaybreak[2]\\         
=    & \|\mY\|_F 
       \left \|  
       \tau_1^{(1)} B_1(\vw_1^{(j(1))}) -\tau_1^{(2)} B_1(\vw_2^{(j(1))}) 
       \right\|_F  + 
       \sqrt{n_0 m} 
       \left| \tau_1^{(2)} \lambda_1^{(2)} - \tau_1^{(1)} \lambda_1^{(1)} \right|    \displaybreak[2]\\
\leq  &  B_\infty    \|\mY\|_F \left| \tau_1^{(1)} - \tau_1^{(2)} \right| +
         \tau_\infty \|\mY\|_F \left \|  B_1(\vw_1^{(j(1))}) - B_1(\vw_2^{(j(1))}) \right\|_F   \\
& \qquad + \tau_\infty \sqrt{n_0 m}   \left| \lambda^{(2)}_l-\lambda^{(1)}_l \right| 
         + \lambda_\infty  \sqrt{n_0 m}  \left| \tau^{(2)}_l-\tau^{(1)}_l \right| \displaybreak[2]\\
= & \tau_\infty \|\mY\|_F \left \|  B_1(\vw_1^{(j(1))}) - B_1(\vw_2^{(j(1))}) \right\|_F 
+(B_\infty \|\mY\|_F + \lambda_\infty \sqrt{n_{\infty} m}) \left| \tau_1^{(1)} - \tau_1^{(2)} \right|  
+ \tau_\infty \sqrt{n_{\infty} m} \cdot \left| \lambda_1^{(1)} - \lambda_1^{(2)} \right|,
\end{align*}
which 
by \eqref{eq:K_L_M_L_O_L_in_proof} 
gives \eqref{eq:estimate_via_K_L_threshold} for
$L=1$, since $\beta_1 = \tau_\infty \|\mY\|_F$ and
$\kappa_1 = \lambda_\infty \sqrt{n_{\infty} m} + B_\infty \|\mY\|_F$ 
(because $Z_0 = 0$) 
and $\varphi_1 =  \sqrt{n_{\infty} m}  \tau_\infty$.

Now we proceed with the induction step,
assuming that the claim holds for some $L \in \N$.
The estimate \eqref{eq:uga} used for the output after $L+1$ layers, combined with the induction hypothesis give us
\begin{align*}
        &  \left\|f^{L+1}_{\mcl{P}_1}(\mY)-f^{L+1}_{\mcl{P}_2}(\mY) \right\|_F \\
\leq    &  
        \alpha
        \left\|f^{L}_{\mcl{P}_1}(\mY)-f^{L}_{\mcl{P}_2}(\mY) \right\|_F 
      +
          \beta_{L+1}
        \left \| B_{L+1}(\vw_2^{(j(l))})  - B_{L+1}(\vw_1^{(j(l))}) \right\|_{2\to 2} \\
        & \quad +
        \kappa_{L+1}
        \left| \tau^{(1)}_{L+1} - \tau^{(2)}_{L+1} \right| 
         +
        \varphi_{L+1} \left| \lambda^{(2)}_{L+1} - \lambda^{(1)}_{L+1} \right| \displaybreak[2]\\
\leq 
    &   \left(\alpha \beta_L + \beta_{L+1} \right) 
        \max_{l \in [L]} 
        \left \| B_l(\vw_1^{(j(l))}) - B_l(\vw_2^{(j(l))}) \right\|_{2 \to 2}  \\
    &   \qquad 
    + \left(\alpha \kappa_L + \kappa_{L+1} \right)
        \left\|  \bm{\tau}^{(1)} -  \bm{\tau}^{(2)} \right \|_\infty  
    + \left( \alpha \varphi_L + \varphi_{L+1} \right)
        \cdot \left\|  \bm{\lambda}^{(1)} - \bm{\lambda}^{(2)} \right \|_\infty,
\end{align*}
so that \eqref{eq:estimate_via_K_L_threshold} holds with the claimed expression for $K_{L+1}$ and thus finishes the proof, since
\[
    K_{L+1} 
= \alpha K_L  +  \beta_{L+1} 
= \alpha \sum_{l=1}^L \beta_l \alpha^{L-l} +  \beta_{L+1} 
=  \sum_{l=1}^{L+1} \beta_l \alpha^{L+1-l}.
\]
and since similar expressions also hold for $O_{L+1}$ and $M_{L+1}$.
\end{proof}

\noindent 
Let us finally provide the Lipschitz bound for the full network in terms of the parameters.

\begin{corollary}\label{cor:Psi_sigma}
For two networks $h_{\mcl{P}_1}, h_{\mcl{P}_2} \in \hypspacelong$ we have
\begin{align*}
&  
\left\| 
h_{\mcl{P}_1} (\mY)) 
- h_{\mcl{P}_2} (\mY))
\right\|_F\\
&  \leq  
\left\|f^L_{\mcl{P}_1} ({\mY})\right\|_F \left\|B_{L+1}(\vw_1^{(j(L+1))}) - B_{L+1}(\vw_2^{(j(L+1))}) \right\| _{2 \to 2} +
\left\|f^L_{\mcl{P}_1} ({\mY}) - f_{\mcl{P}_2}^L(\mY) \right \|_F \\
&  \leq 
(B_\infty K_L + \|\mY \|_F Z_L) \cdot 
\max_{l \in [L+1]} 
\left \| B_l(\vw_1^{(j(l))}) - B_l(\vw_2^{(j(l))}) \right\|_{2 \to 2} \\
& \quad + M_L  \cdot \left\|  \bm{\tau}^{(1)} -  \bm{\tau}^{(2)} \right \|_\infty 
+ O_L  \cdot \left\|  \bm{\lambda}^{(1)} - \bm{\lambda}^{(2)} \right \|_\infty
\end{align*}
with $K_L$, $M_L$ and $O_L$ as given in \eqref{eq:K_L_threshold},
\eqref{eq:M_L_threshold} and \eqref{eq:O_L_threshold}.   
\end{corollary}

\begin{proof}
Using that $\sigma$ is $1$-Lipschitz, and applying the triangle inequality, we obtain
\begin{align*}
&   \left\| 
h_{\mcl{P}_1} (\mY)) 
- h_{\mcl{P}_2} (\mY))
\right\|_F =
\left \|
    \sigma \left(B_{L+1}(\vw_1^{(j(L+1))}) f^L_{\mcl{P}_1} (\mY) \right) - 
    \sigma \left( B_{L+1}(\vw_2^{(j(L+1))}) f_{\mcl{P}_2}^L (\mY) \right) 
    \right \|_F \\
&  \leq  
\left \|B_{L+1}(\vw_1^{(j(L+1))}) f^L_{\mcl{P}_1} ({\mY}) - B_{L+1}(\vw_2^{(j(L+1))}) f^L_{\mcl{P}_2} ({\mY})
 \right \|_F \\
&  \leq  
\left\|B_{L+1}(\vw_1^{(j(L+1))}) f^L_{\mcl{P}_1} ({\mY})-B_{L+1}(\vw_2^{(j(L+1))}) f^L_{\mcl{P}_1} ({\mY})\right\|_F \\
& +
\left\|B_{L+1}(\vw_2^{(j(L+1))}) f^L_{\mcl{P}_1} ({\mY})-
B_{L+1}(\vw_2^{(j(L+1))}) f_{\mcl{P}_2}^L(\mY) \right \|_F  \\ 
&  \leq  
\left\|f^L_{\mcl{P}_1} ({\mY})\right\|_F \left\|B_{L+1}(\vw_1^{(j(L+1))}) - B_{L+1}(\vw_2^{(j(L+1))}) \right\| _{2 \to 2} +
B_\infty
\left\|f^L_{\mcl{P}_1} ({\mY}) - f_{\mcl{P}_2}^L(\mY) \right \|_F, \end{align*}
where we used that $\|B_{L+1}(\vw_1^{(j(L+1))}) \|_{2 \to 2} \leq B_\infty$
by definition of $B_\infty$, see \eqref{eq:standing_assumptions}.
Using the bound \eqref{eq:estimate_via_K_L_threshold} in
Theorem~\ref{thm:perturbation_threshold} for 
$\left\|f^L_{\mcl{P}_1} ({\mY}) - f_{\mcl{P}_2}^L(\mY) \right \|_F$ and that
$
\left\|f^L_{\mcl{P}_1} ({\mY})\right\|_F \leq \|\mY\|_F Z_L$ by \eqref{eq:Zl-bound2}
yields the claimed estimate.
\end{proof}
\subsection{Covering number estimates and proof of the main result}
\label{subsec:covering_nb_est_thresholds}

First, let us already state the  following classical lemma. 
The proof can be found in various sources; as a reference,  
see \cite[Proposition C.3]{foucart_mathematical_2013}.

\begin{lemma}\label{lemma_covering_numbers}
Let $ \varepsilon > 0$ and let $\| \cdot \|$ be a norm on a 
$n$-dimensional vector space $\mcl{V}$. 
Then, for any subset 
$\mcl{U} \subseteq B_\mcl{V} := \{x \in \mcl{V}: \|x\| \leq 1\}$ contained in the unit ball of $\mcl{V}$, we have
\[
\mathcal{N} \left( \mcl{U}, \|\cdot\|, \varepsilon \right) \leq \left(1 + \frac{2}{\varepsilon} \right)^n.
\]
\end{lemma}

\noindent
The next statement on the covering numbers of product spaces is probably well-known.
We provide its proof for convenience.
\begin{lemma}[Covering Numbers of Product Spaces]\label{lemma:covering_number_product_space}
Consider $p$ metric spaces $(\mcl{S}_1, d_1), \dots, (\mcl{S}_p, d_p)$, and positive numbers
$c_1, \dots, c_p$.
We define the product space $\mcl{S}$, equipped with the metric $d$ by
\begin{equation}
    \mcl{S} = (\mcl{S}_1 \times \dots \times \mcl{S}_p, d),
    \qquad
    d(x,y) = \sum_{k=1}^p c_k d_k(x_k,y_k),
\end{equation}
where $x = (x_1, \dots, x_p), y = (y_1, \dots, y_p) \in \mcl{S}$.
Then, we have the covering number estimate
\begin{equation}
     \cover\left(\mcl{S}, d, \varepsilon \right)   
\leq \prod_{k=1}^p   
     \cover\left(\mcl{S}_k, d_k, \varepsilon/(c_k \cdot p) \right).
\end{equation}
\end{lemma}

\begin{proof}
Suppose that, for any $k \in [p]$, we have individual coverings of $\mcl{S}_k$ at level 
$\varepsilon/(c_k p)$ of cardinality $\cover\left(\mcl{S}_k, d_k , \varepsilon/(c_k p) \right)$. We claim that the product of all these $\varepsilon/(c_k p)$-nets is an $\varepsilon$-net for the product space $S$.
Indeed let $x = (x_1, \dots, x_p) \in \mcl{S}$, i.e. $x_k \in \mcl{S}_k$. 
Then, for each $x_k \in \mcl{S}_k$, there exists some element $y_k$ in the 
$\varepsilon/(c_k \cdot p)$-net of $\mcl{S}_k$,
i.e., $d_k(x_k,y_k) \leq \varepsilon/(c_k \cdot p)$. 
Then, $y = (y_1, \dots, y_p)$ is an element of the product of all nets, and by the definition of the metric $d$ it holds 
$d(x,y) \leq c_1(\varepsilon/(c_1 \cdot p)) + \dots + c_p(\varepsilon/(c_p \cdot p) = \varepsilon$.
\end{proof}


In order to bound the integral arising from Dudley's inequality, we will need the following estimate refining \cite[Lemma C.9]{foucart_mathematical_2013}, which is to crude for $\beta$ close to zero below.
\begin{lemma}\label{lem:int-estimate} For $\alpha, \beta > 0$ 
and the function $\Psi$ being defined in \eqref{eq:function_Psi},
it holds
\begin{equation}
\int_0^\alpha \sqrt{\log \left(1+\frac{\beta}{t}\right) } \d t
\leq \alpha \Psi(\beta/\alpha),
\label{eq:integral_inequality_II}
\end{equation}
where
\[
\Psi(t) := \sqrt{\log(1+t) + t(\log(1+t) - \log(t))}.
\]
The function $\Psi$ satisfies $\lim_{t \to 0} \Psi(t) = 0$ and $\Psi(t) \leq \sqrt{\log(e(1+t))}$ for all $t \in \R$.
\end{lemma}
Note that by setting $\Psi(0) = 0$ the above estimates is trivially true also for $\beta=0$.

\begin{proof} 
We proceed similarly to the proof of \cite[Lemma C.9]{foucart_mathematical_2013} and first apply the Cauchy-Schwarz inequality to obtain
\[
\int_0^\alpha \sqrt{\log(1+\beta t^{-1})} \d t \leq \sqrt{ \int_0^\alpha 1 \d t \cdot \int_0^\alpha \log\left(1+ \beta t^{-1}\right) \d t} 
\]
For the second integral on the right hand side above we apply a change of variable and integration by parts to obtain
\begin{align*}
&\int_0^\alpha \log(1+\beta t^{-1}) \d t = \beta \int_{\beta/\alpha}^\infty u^{-2} \log(1+u) \d u\\
&= \beta \left.-u^{-1} \log(1+u) \right|_{\beta/\alpha}^\infty + \beta \int_{\beta\alpha}^\infty u^{-1} \frac{1}{1+u}\d u
= \alpha \log(1+\beta/\alpha) + \beta \lim_{z \to \infty} \left[ \int_{\beta/\alpha}^z \frac{1}{u} \d u - \int_{\beta/\alpha}^z \frac{1}{1+u} \d u \right]\\
& = \alpha \log(1+\beta/\alpha) + \beta\left(\log(1+\beta/\alpha) - \log(\beta/\alpha)\right)
= \alpha \Psi(\beta/\alpha)
\end{align*}
by definition of $\Psi$. A combination
with the inequality derived shows inequality
\eqref{eq:integral_inequality_II}.
Since $\lim_{t \to 0} t \log(t) = 0$ it follows
easily that $\lim_{t \to 0} \Psi(t) = 0$.
Moreover, by the mean-value theorem, for some $\xi \in [t, 1+t]$,
\[
\log(1+t) - \log(t) = \frac{1}{\xi} \leq \frac{1}{t}.
\]
Hence, $\Psi(t) \leq \sqrt{\log(1+t) + 1} = \sqrt{\log(e(1+t))}$.
\end{proof}

Finally, we are prepared for the proof of our main result.

\begin{proof}[Proof of Theorem \ref{theorem:main_result}]
By the assumption \eqref{eq:B_l_Lipschitz} that $B_l$ is $D_l$-Lipschitz,
putting $D_\infty \defeq  \max_{l = 1,\dots, L} D_l$, see \eqref{eq:definition_D_l} for the definition of $D_l$, 
Corollary \ref{cor:Psi_sigma} 
implies that

\begin{align*}
     & \left\| h_{\mcl{P}_1} (\mY) - h_{\mcl{P}_2} (\mY) \right\|_F  \\
\leq & (B_\infty K_L + \|\mY \|_F Z_L) \cdot \max_{l \in [L+1]} 
        \left \| B_l(\vw_1^{(j(l))}) - B_l(\vw_2^{(j(l))}) \right\|_{2 \to 2} \\
 \quad & + M_L  \cdot \left\|  \bm{\tau}^{(1)} -  \bm{\tau}^{(2)} \right \|_\infty 
+ O_L  \cdot \left\|  \bm{\lambda}^{(1)} - \bm{\lambda}^{(2)} \right \|_\infty\\
\leq &  
        (B_\infty K_L + \|\mY \|_F Z_L) \cdot D_\infty \cdot \left\| \mW_1 -  \mW_2 \right \|_\mcl{X} 
      + M_L  \cdot \left\|  \bm{\tau}^{(1)} -  \bm{\tau}^{(2)} \right \|_\infty 
       + O_L  \cdot \left\|  \bm{\lambda}^{(1)} - \bm{\lambda}^{(2)} \right \|_\infty 
\end{align*}

Recalling that $Q_L = (B_\infty K_L + \|\mY \|_F Z_L) \cdot D_\infty$, see \eqref{eq:Q_L_threshold}, we equip $\mcl{Y}=\mcl{T} \times \Lambda \times \mcl{W}$ with the following norm
 \begin{equation}
\left\| ( \bm{\tau}, \bm{\lambda}, \mW) \right\|_\mcl{Y}
\defeq 
M_L  \| \bm{\tau}\|_\infty + 
O_L  \| \bm{\lambda}\|_\infty + Q_L \| \bm{\mW}\|_\mcl{X}, 
\quad (\bm{\tau}, \bm{\lambda}, \mW) \in \mcl{Y}
\label{eq:norm_Y}
\end{equation}
where $\|\cdot\|_\mcl{X}$ was defined in \eqref{eq:max_norm_X}.
Recall from \eqref{thresholds-inclusion}
that $\mathcal{T} \subset \bm{\tau_0} + r_1B_{\|\cdot\|_\infty}^L$ and 
$\Lambda \subset \bm{\lambda}_0 + r_2B_{\|\cdot\|_\infty}^L$,
while $\mcl{W} \subset W_\infty B_{\mcl{X}}^K$ by
\eqref{eq:standing_assumptions}.
Using that covering numbers with respect to norms are invariant under translations of the set,
Lemma~\ref{lemma:covering_number_product_space} 
gives
\begin{align*}
     &  \cover\left( \mcl{M}, \|\, \cdot \,\|_F, \varepsilon \right) 
\leq   \cover  
        \left(
            \mcl{T} \times \Lambda \times \mcl{W}, \|\, \cdot \,\|_\mcl{Y}, \varepsilon 
        \right) \\
\leq &  \cover  
        \left(           r_1B_{\|\cdot\|_\infty}^L,
            \|\, \cdot \,\|_\infty, \varepsilon/(4 \cdot M_L)
        \right) 
        \cdot
        \cover  
        \left(         r_2B_{\|\cdot\|_\infty}^L,
            \|\, \cdot \,\|_\infty, \varepsilon / (4\cdot O_L)
        \right) 
        \cdot
        \cover  
        \left(
               W_\infty B^{K}_\mcl{X},
            \|\, \cdot \,\|_\mcl{X}, \varepsilon /(4\cdot Q_L)
        \right) \\
\leq &  
\left( 1 + \frac{8 r_2 O_L}{\varepsilon} \right)^{L}
\left( 1 + \frac{8 r_1 M_L}{\varepsilon} \right)^{L}
\left( 1 + \frac{8 W_\infty Q_L}{\varepsilon} \right)^{K}
\end{align*}

Already preparing its application in Dudley's integral, let us apply the logarithm to obtain
\begin{align}
& \log\left(\cover\left( \mcl{M}, \|\, \cdot \,\|_F, \varepsilon \right)\right)
 \nonumber\\
\leq &
K \log \left( 1 + \frac{8 W_\infty Q_L}{\varepsilon} \right) +
L \log \left( 1 + \frac{8 r_2 O_L}{\varepsilon} \right) +
L \log \left( 1 + \frac{8 r_1 M_L}{\varepsilon} \right)
\label{eq:covering_numbers_m_threshold}
\end{align}
Plugging the covering number estimate \eqref{eq:covering_numbers_m_threshold} into Dudley's integral (see \eqref{rademacher_2_rewritten_threshold} and \eqref{eq:dudley_bound}) gives
\begin{align*}
         &   \E \sup_{\mM \in \mcl{M}} \frac{1}{\dimS} 
            \sum_{i=1}^\dimS \sum_{k=1}^{n_L} \varepsilon_{ik} M_{ik} 
\leq       \frac{4\sqrt 2}{\dimS}\int_0^{\sqrt{\dimS}\Bout/2}  
            \sqrt{\log \cover(\mcl{M}, \|\cdot\|_F, \varepsilon)} \d\epsilon \\
\leq    &   \frac{4\sqrt{2K}}{\dimS}\int_0^{\sqrt{m}\Bout/2}  
            \sqrt{\log \left( 1 + \frac{8 W_\infty Q_L }{\varepsilon} \right)}
            \d\epsilon 
  +  \frac{4\sqrt{2L}}{\dimS}\int_0^{\sqrt{m}\Bout/2}   
            \sqrt{\log \left( 1 + \frac{8 r_2 O_L}{\varepsilon} \right)}
            \d\epsilon \\         
&\quad     +\frac{4\sqrt{2L}}{\dimS}\int_0^{\sqrt{m}\Bout/2}  
            \sqrt{ \log \left( 1 + \frac{8 r_1M_L}{\varepsilon} \right)}
            \d\epsilon  \\
\leq    &  2\sqrt{2}\Bout
\left[\sqrt{\frac{K}{m}}\Psi\left(\frac{16 W_\infty Q_L}{\sqrt{m}\Bout}\right)
+ \sqrt{\frac{L}{m}}\Psi\left(\frac{8r_2 O_L}{\sqrt{m}\Bout}\right)
+ \sqrt{\frac{L}{m}} \Psi\left(\frac{8r_1 M_L}{\sqrt{m} \Bout} \right)
\right].
\end{align*}
where we applied Lemma~\ref{lem:int-estimate} in the last step.
The theorem is obtained using Theorem \ref{thm:ge_vs_rademacher} 
and Lemma \ref{lem:Maurer}, our application of \eqref{rademacher_2}.
\end{proof}

\section{Numerical Results}
\label{sec:numerical}

In this section, we report on numerical experiments performed to examine the obtained generalization bounds in comparison with the actual generalization error. 
Note that we do not necessarily aim to achieve state of the art results in terms of reconstruction, but instead we pursue two main goals in this section. 


On the one hand, we would like to illustrate that the proposed framework is meaningful by capturing various interesting examples of practical interest.
On the other hand, we are interested in the generalization error and its scaling with respect to the training parameters. Specifically, we have obtained a sample complexity bound that holds uniformly over the hypothesis space and for any distribution. Although the bound is quite simple and general, it is interesting to see if we can expect improvements when it is applied to data from low complexity distributions. ISTA is used mainly in sparse coding and recovery, and therefore we consider such a scenario.

The experiments are run over a synthetic and the MNIST data set \cite{lecun1998mnist} using Pytorch implementation and Titan XP GPU. In all the experiments, we have used the Adam optimizer for training the network with the learning rate $10^{-2}$. The objective function for optimization is the MSE loss of the recovered vector with respect to the ground truth. For the synthetic data, we use the training data with size 10000 and the test date with size 50000. The model is trained for 10 epochs. We explain the experiment details here. First, the measurement matrix is chosen as a random Gaussian matrix. Soft thresholding algorithms require constraints on the spectral norm of the measurement matrix for convergence. We guarantee this by proper normalization. 
For the synthetic data set, we have chosen the classical compressed sensing setup with sparse vectors. We run our experiments over different input and output dimensions and sparsity order. We choose a random orthogonal matrix as the ground truth dictionary. We initialize the model with a random matrix.  
To generate sparse vectors, the support is chosen uniformly at random. The non-zero values are drawn from the standard normal distribution. We repeat the experiments for the synthetic data between 10 to 100 times to obtain a smoother curve after averaging. 

For the first experiment, we enforce an (almost) orthogonality constraint on the learned dictionary by adding a regularization term  $\|\mI - \mPhi^\top \mPhi\|_{2 \to 2}$ to the loss function. See \cite{lezcano2019cheap} for a different approach using exponential mapping from the skew-symmetric matrices onto the special orthogonal group $SO(N)$.



\subsection{MNIST experiment}
In 
Figures~\ref{fig:abserror_MNIST} and \ref{fig:generror_MNIST}, we present the results for the MNIST dataset. It can be seen that  the original images can be recovered with only a small number of layers. The error in the MNIST experiments is the pixel-based error normalized by the image dimension. The trained model is compared with the iterative soft thresholding algorithm (ISTA)  using a similar structure but with 5000 iterations. 
\begin{figure}[th]
    \centering
    \begin{subfigure}[b]{0.47\textwidth}
\includegraphics[width=\textwidth]{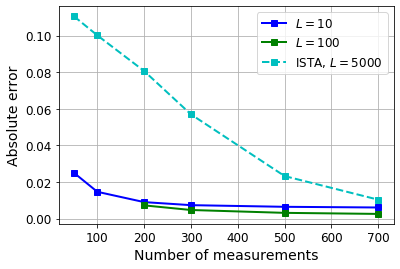}
        \caption{Absolute reconstruction error for different measurements of MNIST}
        \label{fig:abserror_MNIST}
    \end{subfigure}
\centering
\hfill
    \begin{subfigure}[b]{0.47\textwidth}
\includegraphics[width=\textwidth]{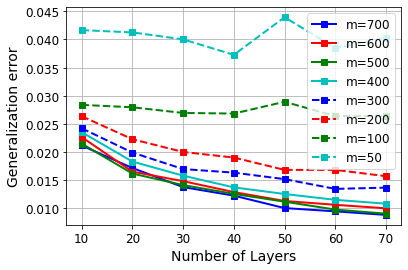}
        \caption{GE for different measurements of MNIST}
        \label{fig:generror_MNIST}
    \end{subfigure}
\caption{MNIST dataset}
\label{fig:mnist_dataset}
\end{figure}

The generalization error is plotted in Figure \ref{fig:generror_MNIST}. First of all, it can be seen that the generalization error decreases with increasing the number of measurements. A similar observation is made in the experiments on the synthetic data. Moreover, increasing the number of layers decreases the generalization error. 

While our theoretical bound actually increases with increasing number of layers (and slightly increases with increasing number of measurements \cite{behboodi2020generalization}, although that dependence is swallowed
by the constant in \eqref{eq:rough_GE_main_example}), the better behavior obtained here, may be justified from a compressive sensing standpoint. The reconstruction task becomes easier with 
more measurements and the quality becomes better with more iterations.
Additional assumptions that may take the specific compressive sensing
scenario into account are currently not captured by our general worst case result Theorem~\ref{theorem:main_result}, which provides a uniform complexity bound that applies to all possible input distributions.
We conjecture that the bound of Theorem~\ref{theorem:main_result} can
be improved by taking into account assumptions like
to sparsity of the input and properties of the measurement matrix $\mA$ and the  underlying true dictionary $\Phi_0$ such as a restricted isometry property of $\mA \Phi_0$. But presently, it is not clear how this could potentially be done.



\subsection{Experiment on synthetic data}


In Figure \ref{fig:generror_synth_n}, the generalization error is plotted for a variation of parameter choices. The input dimension is fixed to $N=120$. We have used a linear fit between the data points with different numbers of layers. 
Similar to the MNIST experiments, the generalization error decreases with the number of measurements $n$. 
Increasing the number of layers, however, increases the generalization error. 
This is now in accordance with our theoretical results and suggests that
the logarithmic scaling in $L$ may not be removed in general.

\begin{figure}[th]
    \centering
\includegraphics[width=\textwidth]{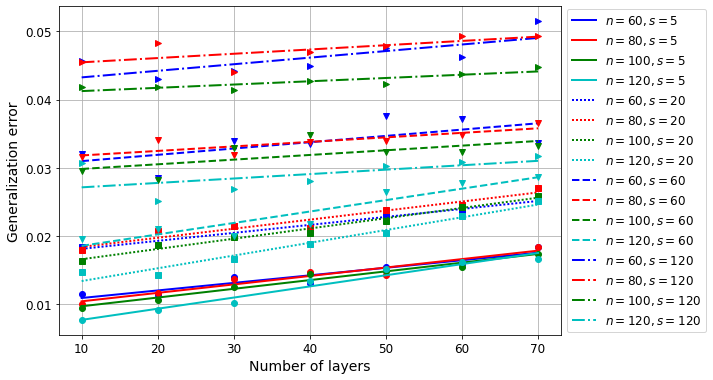}
        \caption{Generalization error for $N=120$}
        \label{fig:generror_synth_n}
\end{figure}

\subsection{Synthetic dataset: variations in the architecture}
In this section, we repeat previous experiments with small changes in the architecture to see the correlation with our generalization bound. In many cases, our bound correlates well with the generalization error, although, as we mentioned above, we expect improvements by considering the input structure, for instance sparsity, in our analysis. The scenario of learning thresholds and using convolutional dictionaries shows bigger discrepancy with our generalization bound.

\subsubsection{Non-orthogonal dictionaries} 

We first consider the case where the dictionaries are chosen to be an arbitrary matrix and not necessarily orthogonal. 
A similar case has already been studied in \cite{kouni2021admm}, where a neural network jointly learns a decoder for analysis-sparsity-based compressive sensing, and an overcomplete sparsifying dictionary. Their framework is tested on image and speech datasets.
Here, in order to evaluate how close our theoretical bounds are to reality, Figure~\ref{fig:ge_nonorth}
plots the empirically observed generalization error versus 
our theoretical generalization bound. We clearly observe that
our bounds are generally positively correlated with the empirical generalization error.  Indeed, the generalization error increases with the number of layers and with $N$. The other dependencies are less clear, since their effect is mixed with other terms in the generalization bound. We have chosen a sparsity $s=10$ for these experiments and plotted the generalization bound from Theorem~\ref{theorem:main_result}.

\begin{figure}
    \centering
\includegraphics[width=\textwidth]{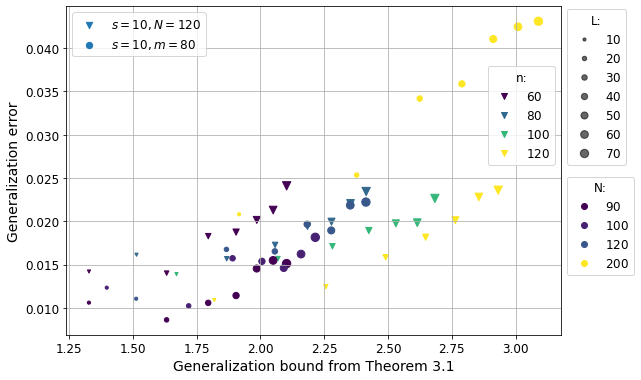}
        \caption{Generalization error vs generalization bound comparison for non-orthogonal dictionary}
        \label{fig:ge_nonorth}
\end{figure}

\subsubsection{Learned thresholds}
In Figure \ref{fig:ge_learnedth}, we show the result for the case where the thresholds are learned, and the dictionaries are not orthogonal. 
It is difficult to spot a general trend in the the empirical generalization error as a function of depth and input dimension. The same conclusion holds for convolutional dictionaries, as we will see later.  
However, it is also important to notice that in both cases, the generalization error is significantly smaller compared to all other scenarios. This means that learning thresholds and using convolutional dictionaries yield the best generalization error among other cases.  
One can see that the error tends to grow with the network dimensions (depth, input dimension), but there are many outliers, and the curves do not smooth out even after many iterations. 
\begin{figure}
    \centering
\includegraphics[width=\textwidth]{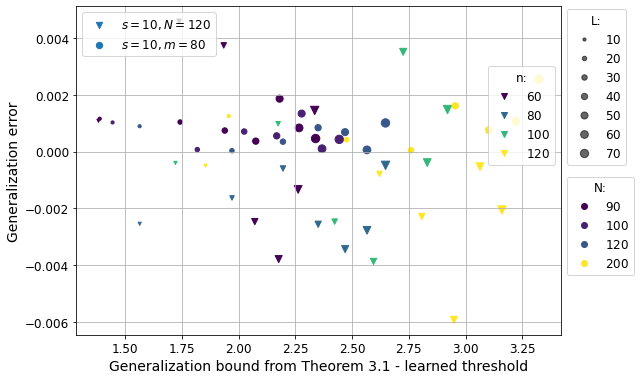}
        \caption{Generalization error vs generalization bound comparison for learned thresholds}
        \label{fig:ge_learnedth}
\end{figure}

\subsubsection{Alternating dictionaries}
We repeat the analysis for the case where two different dictionaries $\mPhi_1$ and $\mPhi_2$ are used in the model. For odd layers, the first dictionary is used and for the even layers  the second dictionary. In other words, we have:
\begin{align*}
f_{2l-1} \left( \vz, \vy \right)
&=
S_{\tau \lambda}
\left[ \left(\mI_{N} - 
    \tau (\mA\mPhi_1)^\top  (\mA\mPhi_1) 
    \right )\vz 
     + \tau  (\mA\mPhi_1)^\top \vy \right]\\
     f_{2l} \left( \vz, \vy \right),
&=
S_{\tau \lambda}
\left[ \left(\mI_{N} - 
    \tau (\mA\mPhi_2)^\top  (\mA\mPhi_2) 
    \right )\vz 
     + \tau  (\mA\mPhi_2)^\top \vy \right].
\end{align*}
In this scenario, the number of learnable parameters are doubled. The training process is done similarly to previous cases. As it can be seen in Figure \ref{fig:ge_multidict}, our generalization bound tends to be positively correlated with the true generalization error. A general trend in this figure is that the generalization error increases with the dimensions $N$ and $L$ as darker and smaller points tend toward the left corner. Note that the generalization error does not show any strong correlation with the number $\dimS$ of measurements. 
\begin{figure}
    \centering
\includegraphics[width=\textwidth]{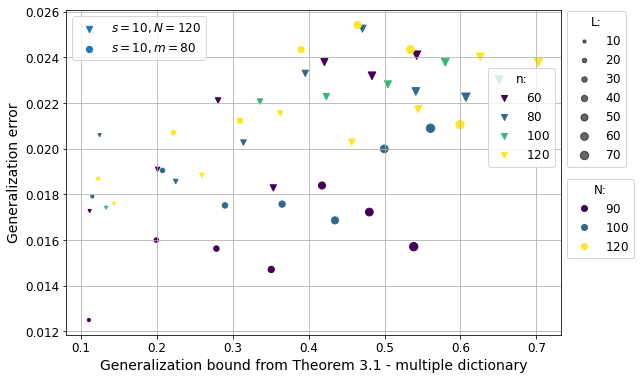}
        \caption{Generalization error vs generalization bound comparison for alternating dictionaries}
        \label{fig:ge_multidict}
\end{figure}

\subsubsection{Convolutional and alternating dictionaries}
In this part, we use convolutional dictionaries. The convolutional dictionary is shared between all the layers. The effective number of parameters in this network would be the size of the convolutional kernel, which is chosen equal to $7$. The numerical results are presented in Figure \ref{fig:ge_conv}. Similar to the learned threshold case, there is definite trend between the generalization error and network dimensions, although it can be seen that there is an increasing trend of the error with depth and input dimension. Note that the empirical generalization error is very small in this case.

\begin{figure}
    \centering
\includegraphics[width=\textwidth]{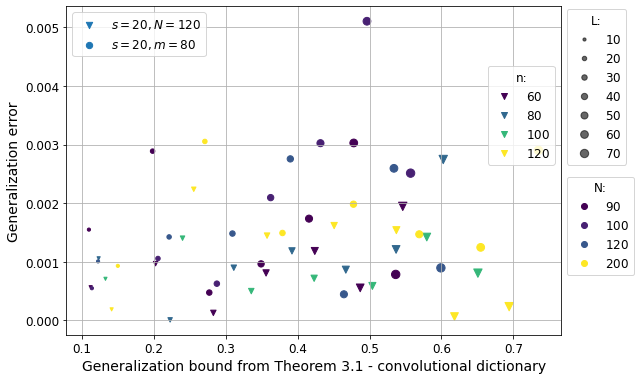}
        \caption{Generalization error vs generalization bound comparison for convolutional dictionaries}
        \label{fig:ge_conv}
\end{figure}
\newpage
\bibliographystyle{imaiai}
\bibliography{references}

\end{document}